\documentclass[letterpaper]{article} 
\usepackage[submission]{aaai25}  
\usepackage{times}  
\usepackage{helvet}  
\usepackage{courier}  
\usepackage[hyphens]{url}  
\usepackage{graphicx} 
\usepackage{natbib}  
\usepackage{caption} 
\usepackage{algorithm}
\usepackage{algorithmic}

\usepackage{microtype}
\usepackage{graphicx}
\usepackage{subfigure}
\usepackage{booktabs} 

\usepackage{url}            
\usepackage{booktabs}       
\usepackage{amsfonts}       
\usepackage{nicefrac}       
\usepackage{microtype}      
\usepackage{comment}
\usepackage{graphicx} 
\usepackage{pdflscape}
\usepackage{bm}
\usepackage{soul}
\usepackage{listings}
\usepackage{multirow}
\usepackage{multicol}

\usepackage{amsmath}
\usepackage{amssymb}
\usepackage{mathtools}
\usepackage{amsthm}

\usepackage{color,soul}

\usepackage{newfloat}
\usepackage{listings}
\DeclareCaptionStyle{ruled}{labelfont=normalfont,labelsep=colon,strut=off} 
\floatstyle{ruled}
\newfloat{listing}{tb}{lst}{}
\floatname{listing}{Listing}
\newcommand\Tstrut{\rule{0pt}{1.5ex}}         

\theoremstyle{plain}
\newtheorem{theorem}{Theorem}

\theoremstyle{definition}

\theoremstyle{remark}

\title{A Stochastic Approach to Bi-Level Optimization for\\Hyperparameter Optimization and Meta Learning}
\author{
    Written by AAAI Press Staff\textsuperscript{\rm 1}\thanks{With help from the AAAI Publications Committee.}\\
    AAAI Style Contributions by Pater Patel Schneider,
    Sunil Issar,\\
    J. Scott Penberthy,
    George Ferguson,
    Hans Guesgen,
    Francisco Cruz\equalcontrib,
    Marc Pujol-Gonzalez\equalcontrib
}
\affiliations{
    \textsuperscript{\rm 1}Association for the Advancement of Artificial Intelligence\\

    1101 Pennsylvania Ave, NW Suite 300\\
    Washington, DC 20004 USA\\
    proceedings-questions@aaai.org
}

\title{A Stochastic Approach to Bi-Level Optimization for\\Hyperparameter Optimization and Meta Learning}
\author{Minyoung Kim$^1$ \ \& \ Timothy M.~Hospedales$^{1,2}$ \\
$^1$Samsung AI Center Cambridge, UK \ \ \ \ \ \ \ \ \ \ \ \ $^2$University of Edinburgh, UK \\
\ \texttt{mikim21@gmail.com} \ \ \ \ \ \ \ \ \ \ \ \ \ \ \ \ \ \ \ \ \ \ \ \ \ \ \ \ \ \ \  \texttt{t.hospedales@ed.ac.uk}
}

\title{A Stochastic Approach to Bi-Level Optimization for\\Hyperparameter Optimization and Meta Learning}
\author{Minyoung Kim$^1$ \ \ \ \ \ \ \ \ \ \ \ \ \ \ Timothy M.~Hospedales$^{1,2}$ \vspace{+0.5em}\\ 
{\small 
\texttt{mikim21@gmail.com} \ \ \ \ \ \ \ \ \ \ \ \ \ \ \ \ \ \ \ \ \texttt{t.hospedales@ed.ac.uk} \ \ \ \ \ \ \ \ \ \ \\
$^1$Samsung AI Center Cambridge, UK \ \ \ \ \ \ \ \ \ $^2$University of Edinburgh, UK \ \ \ \ \ \ \ }
}

\usepackage{bibentry}

\begin{document}

\maketitle

\begin{abstract}
We tackle the general differentiable meta learning problem that is ubiquitous in modern deep learning, including hyperparameter optimization, loss function learning, few-shot learning, invariance learning and more. These problems are often formalized as Bi-Level optimizations (BLO). We introduce a novel perspective by turning a given BLO problem into a stochastic optimization, where the inner loss function becomes a smooth probability distribution, and the outer loss becomes an expected loss over the inner distribution. To solve this stochastic optimization, we adopt Stochastic Gradient Langevin Dynamics (SGLD) MCMC to sample inner distribution, and propose a recurrent algorithm to compute the MC-estimated hypergradient. Our derivation is similar to forward-mode differentiation, but we introduce a new first-order approximation that makes it feasible for large models without needing to store huge Jacobian matrices. The main benefits are two-fold: i) Our stochastic formulation takes into account uncertainty, which makes the method robust to suboptimal inner optimization or non-unique multiple inner minima due to overparametrization; ii) Compared to existing methods that often exhibit unstable behavior and hyperparameter sensitivity in practice, our method leads to considerably more reliable solutions. We demonstrate that the new approach achieves promising results on diverse meta learning problems and easily scales to learning 87M hyperparameters in the case of Vision Transformers.
\end{abstract}

%

\section{Introduction}\label{sec:intro}

We consider general differentiable meta learning, which includes bi-level optimization (BLO) problems such as hyperparameter optimization (HPO) and few-shot meta-learning.
We introduce our problem setting, terminology and notation through a representative BLO problem of learning neural network regularizer hyperparameters (aka weight decay, or {\em L1/2}-regularization). Expressed as a BLO, this is
\vspace{+0.3em}
\begin{align}
\min_{\lambda} & \ \mathbb{E}_{(x,y)\sim V}[ l_V(x,y;\theta^*(\lambda)) ] \label{eq:blo_reg} \\
& \textrm{s.t.} \ \theta^*(\lambda) = \arg\min_\theta \mathbb{E}_{(x,y)\sim T}[ l_T(x,y;\theta) ] + R(\lambda,\theta), \nonumber
\end{align}
where $\theta$ are neural network parameters to learn, $\lambda$ are all hyperparameters to learn (often denoted as outer variables or hyperparameters, depending on the context), $T$ and $V$ indicate training and validation data sets, respectively, $l_T$ and $l_V$ indicate training and validation losses, and $R$ is the regularization or weight decay term. In the parameter-wise regularization case, $R(\lambda,\theta) = \sum_j \lambda_j \theta_j^2$ or $\sum_j \lambda_j |\theta_j|$ in which $\lambda$ has the same structure as $\theta$.

The main challenge in differentiable BLO is to efficiently and accurately compute the hypergradient $dl_V/d\lambda$ to update the hyperparameters. %
In the literature there exist some well-known approaches. Some aim to unroll the inner-loop optimization (i.e., forward- or reverse-mode differentiation \cite{fmd_rmd}, denoted by FMD and RMD, respectively), while others aim to compute it via the implicit function theorem (IFT) to circumvent the infeasible memory cost of saving intermediate Hessian matrices (in FMD) or computation graphs (in RMD) from unrolled inner optimization steps.  Despite the theoretical promise of IFT-based meta-gradient \cite{imaml,ift} for memory-efficiency, it still suffers from a difficult and potentially unstable Hessian inverse, and it is highly reliant on the \emph{gradient-equal-to-zero} (i.e., perfectly converged) condition for the inner optimization. In practice these drawbacks mean that it is unreliable and hard to tune. 


A fundamental issue for all these existing BLO solutions is that the inner optimization is deterministic, and they are unable to account for the uncertainty in the inner problem and/or its solution. When applied to contemporary deep learning there are at least two highly practical examples of inner optimization uncertainty: When implemented in practice  by minibatch SGD with a finite number of steps, the inner loop is not perfectly converged (violating IFT's assumption); and when applied to non-convex neural networks, the inner optimization has many local minima due to overparametrization~\cite{overparam_blo}, leading to uncertainty over which inner minima is returned. 

To address these issues, we introduce a new stochastic gradient approach to hyper-gradient calculation for BLO. We generalize the standard deterministic BLO problem to a stochastic optimization where the inner optimization produces a smooth probability distribution, and the outer optimization takes an expectation with respect to the inner posterior. 
This stochastic generalization allows us to deal with uncertainty in the inner optimization arising from noise, the use of minibatch SGD, or from multiple local minima. More specifically, we exploit SGLD to obtain posterior samples of the inner optimization posterior, and derive a new efficient hypergradient estimation algorithm by taking derivatives of the SGLD step equations.

\section{Proposed Algorithm}\label{sec:main}

We first introduce {\em stochastic optimization}, and show that it generalizes 
the BLO problem. We argue that solving the stochastic optimization is preferable to solving the deterministic BLO due to better handling of inner optimization uncertainty, which we illustrate via a failure case of BLO in Sec.~\ref{sec:toy_poly_regr}. 
Our proposed algorithm for solving the stochastic optimization is described in Sec.~\ref{sec:sgld_recursion}, in which new recursions for meta gradients are derived from stochastic gradient MCMC (Langevin dynamics).

\subsection{Stochastic optimization (SO) 
}\label{sec:stoch_optim}

We aim to solve the {\em stochastic optimization} problem:
\begin{align}
\min_\lambda \ \mathbb{E}_{p(\theta|\lambda)}[f(\lambda,\theta)] \ \ \textrm{s.t.} \ \ 
p(\theta|\lambda) = \frac{\exp(-E(\lambda,\theta))}{Z(\lambda)},
\label{eq:stoch_optim}
\end{align}
where $f(\lambda,\theta)$ and $E(\lambda,\theta)$ are the given scalar functions for the problem, and $Z(\lambda)$ is the normalizer that ensures $p(\theta|\lambda)$ to be a proper distribution. 
We have no restriction or constraints on the problem instance functions $f()$ and $E()$ except for the minimal assumption that evaluating these functions and taking their derivatives are straightforward and easy. We will see that this stochastic optimization formulation is quite flexible and general enough to encompass the general Bi-Level optimization (BLO) problem, that often arises in deep learning, as a special case. 
We will see this connection in detail in what follows.

\noindent$\clubsuit$~\textbf{BLO as a special case of Stochastic optimization.}
The Bi-Level optimization (BLO) problem, despite its several different forms, can be formally and concisely expressed as:
\vspace{+0.3em}
\begin{align}
\min_\lambda \mathcal{L}_V(\lambda,\theta^*(\lambda)) \ \ \textrm{s.t.} \ \ \theta^*(\lambda) \underbrace{=}_{\textrm{or}\ \in} \arg\min_\theta \mathcal{L}_T(\lambda,\theta).
\label{eq:blo}
\end{align}
We often call $\mathcal{L}_V()$ the {\em outer objective} and $\mathcal{L}_T()$ the {\em inner objective}. In typical {\em hyperparameter optimization} (HPO) problems, $\mathcal{L}_V()$ and $\mathcal{L}_T()$ correspond to validation and training losses respectively. Meanwhile $\lambda$ is the set of hyperparameters to learn and $\theta$ are the main parameters (weights of the neural network).
We will show that this (deterministic) BLO problem in (\ref{eq:blo}) is a special case of the stochastic optimization in (\ref{eq:stoch_optim}). First we let the objective function $f(\lambda,\theta)$ in (\ref{eq:stoch_optim}) be equal to $\mathcal{L}_V(\lambda,\theta)$, that is,
\begin{align}
f(\lambda,\theta) := \mathcal{L}_V(\lambda,\theta),
\end{align}
and define the distribution $p(\theta|\lambda)$ as:
\begin{align}
p(\theta|\lambda)\!:=\!\frac{e^{-\mathcal{L}_T(\lambda,\theta)/\tau}}{Z_\tau(\lambda)}, \ \ 
Z_\tau(\lambda)\!=\!\int\!e^{-\mathcal{L}_T(\lambda,\theta)/\tau} \ d\theta,
\end{align}
and $\tau\!>\!0$ is a hyperparameter known as the {\em temperature}. In other words, we set the energy function of $p(\theta|\lambda)$ to be:
\begin{align}
E(\lambda,\theta) := \mathcal{L}_T(\lambda,\theta)/\tau
\label{eq:energy_train_loss}
\end{align}
Through the hyperparameter $\tau$ we can control {\em how certain (or uncertain)} the impact of the inner optimization is on the outer optimization. As an extreme case, it is obvious that as we tend $\tau\!\to\!0$, we make $p(\theta|\lambda)$ converge to the delta function $\delta(\theta\!-\!\theta^*(\lambda))$, and thus the stochastic problem (\ref{eq:stoch_optim}) coincides with (\ref{eq:blo}).



\subsection{Why is SO better than deterministic BLO?}\label{sec:toy_poly_regr}

BLO's deterministic nature inner optimization might be problematic in certain situations. In this section we aim to highlight this issue via an illustrative example where SO succeeds and BLO fails.

Mainstream BLO algorithms used in deep learning rely on {\em one single} optimal solution $\theta^*(\lambda)$ for the inner optimization given $\lambda$, even though the inner optimization problem can be noisy (e.g., noise in training data) and/or can have multiple different local/global optima. Note that the latter situation almost always arises for models with a high degree of redundancy (i.e., overparametrization) such as deep neural networks. The worst scenario is that a (deterministic) BLO algorithm unfortunately selectings an optimum $\theta^*(\lambda)$ that leads to a poor (high) outer loss at each outer iteration, which can eventually lead us to select a poor outer variable $\lambda$ as the final solution of the BLO. 

On the other hand, our stochastic optimization formulation in (\ref{eq:stoch_optim}) essentially takes into account {\em all} $\theta$ as possible solutions to the inner optimization problem through a probability distribution $p(\theta|\lambda)$. It is also beneficial in that the solutions become more robust to potential noise that may reside in the inner problem such as noisy training data, or solution by minibatch-SGD, since we deal with the {\em average} outer loss over all possible $\theta$'s instead of a single-outcome loss.

\begin{figure*}[t!]
\vspace{-1.5em}
\begin{center}
%
\centering
\includegraphics[trim = 3mm 0mm 15mm 3mm, clip, scale=0.22]{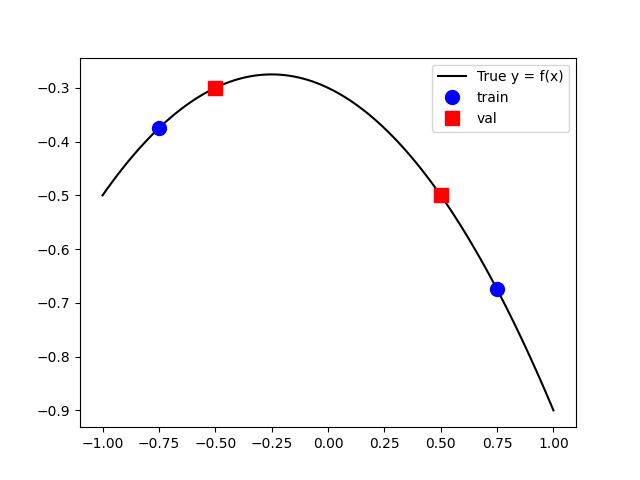} 
\ \   
\includegraphics[trim = 3mm 0mm 10mm 14mm, clip, scale=0.22]{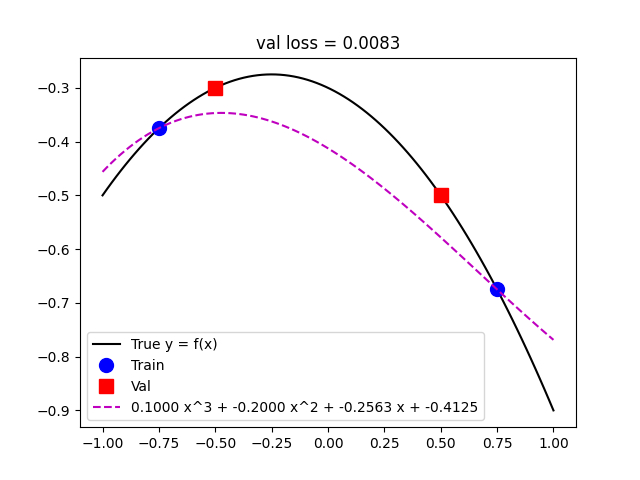} 
\includegraphics[trim = 18mm 0mm 10mm 14mm, clip, scale=0.22]{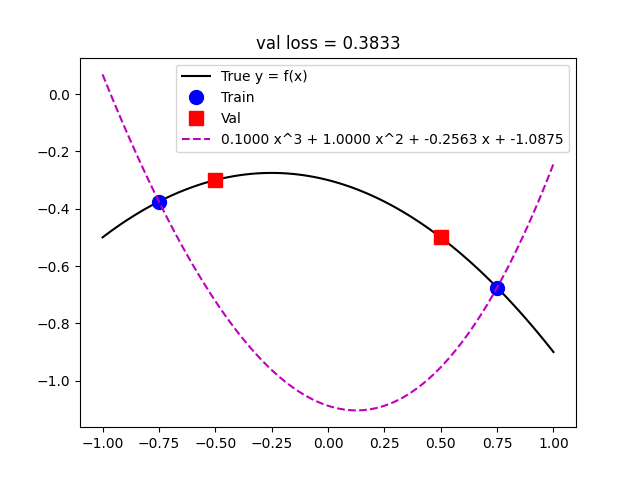} 
\ \ 
\includegraphics[trim = 8mm 3mm 65mm 18mm, clip, scale=0.182]{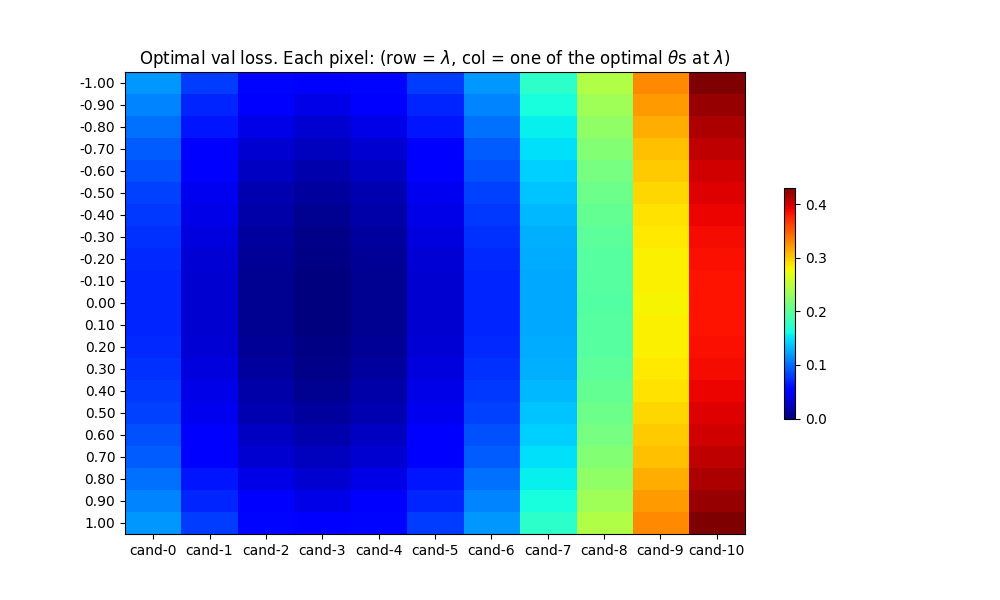} 
\ \  
\includegraphics[trim = 8mm 3mm 65mm 18mm, clip, scale=0.182]{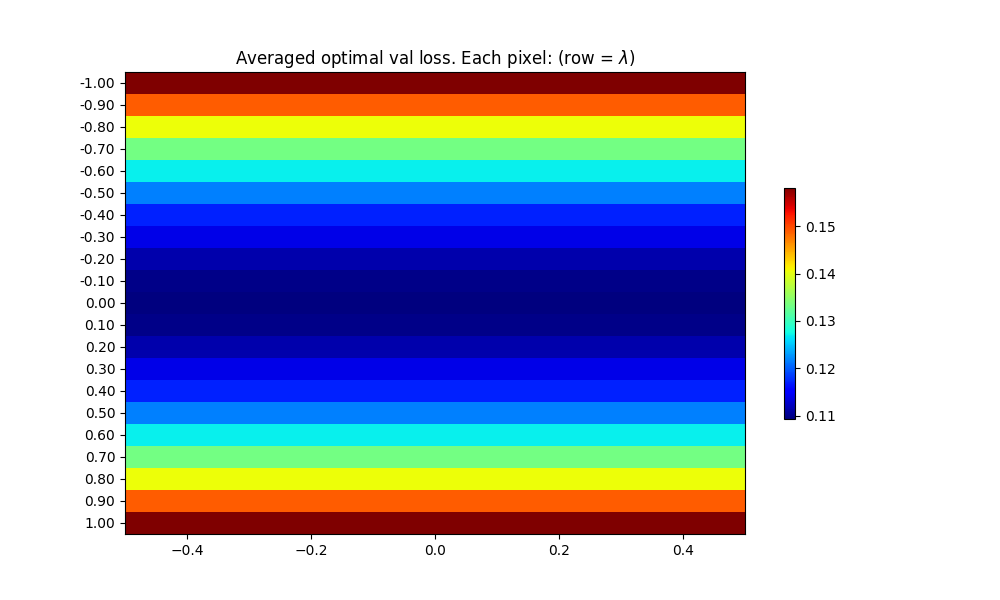} 
\\ \vspace{-0.5em}
\ \ \ \ \ (a) \ \ \ \ \ \ \ \ \ \ \ \ \ \ \ \ \ \ \ \ \ \ \ \ \ \ \ \ \ \ \ \ \ \ \ \ \ \ \ \ \ \ \ \ \ \ \ \ \ \ \ (b) \ \ \ \ \ \ \ \ \ \ \ \ \ \ \ \ \ \ \ \ \ \ \ \ \ \ \ \ \ \ \ \ \ \ \ \ \ \ \ \ \ \ \ \ \ \ \ \ \ \ \ \ \ (c) \ \ \ \ \ \ \ \ \ \ \ \ \ \ \ \ \ \ \ \ \ \ \ \ \ \ \ \ \ \ \ \ \ \ \ \ (d)
\end{center}
\vspace{-1.5em}
\caption{Illustrative toy problem. 
(a) Training and validation data. 
(b) Two $\theta^*(\lambda)$ solutions at $\lambda\!=\!0.1$; (Left) Good $\theta^*(\lambda)$ (val loss 0.0083), (Right) Poor $\theta^*(\lambda)$ (val loss 0.3833).
(c) $\mathcal{L}_V(\lambda,\theta^*(\lambda))$. Each row $\!=\!\lambda$, each column $\!=\!$ one of $\theta^*(\lambda)$. Blue (red) indicates low (high) loss value. 
(d) Row-wise average validation losses $\mathbb{E}_{p(\theta|\lambda)}[\mathcal{L}_V(\lambda,\theta)]$ employed in our proposed stochastic optimization (SO). 
}
\vspace{-1.5em}
\label{fig:toy_poly_regr}
\end{figure*}

To illustrate this benefit clearly and explicitly, we present a toy scenario of the polynomial function regression hyperparameter learning problem. 
First we consider the quadratic $f_{true}(x)\!=\!-0.4 x^2\!-\!0.2 x\!-\!0.3$ as the true data generating function on $x\in[-1,1]$, from which we sample four points: $T\!=\!\{(-0.75,-0.375), (0.75,-0.675)\}$ constituting the training data, and $V\!=\!\{(-0.5,-0.3), (0.5,-0.5)\}$ the validation data (See Fig.~\ref{fig:toy_poly_regr}(a)). 

Our model is assumed to have a cubic form, $f(x;\lambda,\theta)\!=\!\lambda x^3 + \theta_2 x^2 + \theta_1 x + \theta_0$ where $\lambda \!\in\![-1,1]$ is the hyperparameter that decides the degree/smoothness of the polynomial, and $\theta=[\theta_0,\theta_1,\theta_2]^\top\!\in\![-1,1]^3$ are the main parameters. We define the losses as: $\mathcal{L}_V(\lambda,\theta) = \sum_{(x,y)\in V} (f(x;\lambda,\theta)\!-\!y)^2$ and $\mathcal{L}_T(\lambda,\theta) = \sum_{(x,y)\in T} (f(x;\lambda,\theta) - y)^2$ for the validation and train dataset $V$ and $T$, respectively.

With the BLO formulation, we can solve it by an exhaustive {\em tabular} method. Namely, for each candidate $\lambda$ we find {\em all} possible inner optimal solutions $\theta^*(\lambda)$, which are infinitely many, but can be enumerated for some grid of values. Then for each candidate $(\lambda,\theta^*(\lambda))$ we evaluate the outer loss value $\mathcal{L}_V(\lambda,\theta^*(\lambda))$. This is depicted in Fig.~\ref{fig:toy_poly_regr}(c). 
We can see that for each row (given $\lambda$), there are many different $\theta^*(\lambda)$ that yield different outer losses $\mathcal{L}_V$ (cooler colors for lower $\mathcal{L}_V$ than warmer). Note however that they all attain perfect inner loss $\mathcal{L}_T\!=\!0$ (i.e., \emph{all are global inner optima}). Meanwhile, by inspection, can see that the optimal $\lambda$ has to be $0$ (thus quadratic), and 
$\theta^*(0)$ (one of the fourth column in Fig~\ref{fig:toy_poly_regr}(c)) identifies the true quadratic $f_{true}$. 
Examples of good and poor $\theta^*(\lambda)$ at $\lambda\!=\!0.1$ are shown in Fig.~\ref{fig:toy_poly_regr}(b). 

Now, consider typical deterministic BLO solutions, which select just {\em one of the inner optima} $\theta^*(\lambda)$ for each $\lambda$. In Fig.~\ref{fig:toy_poly_regr}(c), we can think of them as {\em randomly} selecting a column for each row $\lambda$. Then the final solution $\lambda$ is chosen as the one with the smallest validation loss among those randomly selected. So, there is a high chance that the final solution is not the optimal one $\lambda\!=\!0$ (e.g., when a good $\theta^*(\lambda)$ is chosen for some suboptimal $\lambda\!\neq\!0$, but a poor $\theta^*(\lambda)$ is chosen for $\lambda\!=\!0$). 
On the other hand, the stochastic optimization  considers the average $\mathbb{E}_{p(\theta|\lambda)}[\mathcal{L}_V(\lambda,\theta)]$ for each $\lambda$, and selects the $\lambda$ with the smallest average. This is visualized in Fig.~\ref{fig:toy_poly_regr}(d), where each row ($\lambda$) has an averaged validation loss. Clearly it is guaranteed to select the optimal $\lambda\!=\!0$ as the final solution. 

When we ran these two approaches on the tabularized BLO problem with 21 (11) equally spaced $\lambda$ ($\theta^*(\lambda)$) grid point values, the deterministic BLO solution via the random column selection was able to find the optimal $\lambda\!=\!0$ for only three out of 20 different runs\footnote{We emphasize that this result has nothing to do with the quality of hypergradients, because it is a grid search. So previous variance-reduction techniques~\cite{pmlr-v162-yang22g} for hypergradients would not address the issue.
}, whereas stochastic optimization always found the optimal solution.

\subsection{Proposed Solution: SGLD Gradient Recursion}\label{sec:sgld_recursion}

We propose a novel method to solve the stochastic optimization (\ref{eq:stoch_optim}). 
Using the Monte-Carlo approximation, 
\begin{align}
\mathbb{E}_{p(\theta|\lambda)}[f(\lambda,\theta)] \approx \frac{1}{M}\!\sum_{m=1}^M\!f(\lambda, \theta^{m}), \ \theta^{m}\!\sim\!p(\theta|\lambda),
\label{eq:mc_estimate}
\end{align}
where $M$ is the number of MC samples. 
The (approximate) hypergradient of the objective can be written as:
\begin{align}
\frac{d}{d\lambda} \mathbb{E}_{p(\theta|\lambda)}[f(\lambda,\theta)] \approx \frac{1}{M} \sum_{m=1}^M \frac{d f(\lambda, \theta^{m})}{d \lambda}, 
\label{eq:dEf_dlambda}
\end{align}
and using the chain rule ($d$ for total, $\partial$ for partial gradients),
\begin{align}
\frac{d f(\lambda, \theta^{m})}{d \lambda} = 
\frac{\partial f(\lambda, \theta^{m})}{\partial \lambda} + \frac{\partial f(\lambda, \theta^{m})}{\partial \theta} \cdot \frac{d \theta^{m}}{d \lambda}.
\label{eq:df_dlambda}
\end{align}
Note that each sample $\theta^{m}$ is a function of $\lambda$. However, in general it may be difficult to express $\theta^{m}$ explicitly in terms of $\lambda$. 

We instead adopt the stochastic-gradient Langevin dynamic  SGLD~\cite{sgld}, which repeats the recurrence:
\begin{align}
\theta \ \leftarrow \ \theta + \frac{\epsilon}{2} \frac{\partial \log p(\theta|\lambda)}{\partial \theta} + \sqrt{\epsilon} z,
\label{eq:sgld}
\end{align}
where $z\!\sim\!\mathcal{N}(0,\kappa^2 I)$ 
and $\epsilon$ is a small (ideally infinitesimal) step size. We assume that $\epsilon^{1+\alpha}$ is virtually $0$ for $\alpha\!\geq\!0.5$ (e.g., $\epsilon^2\!\simeq\!0$ and $\epsilon^{1.5}\!\simeq\!0$), which is reasonable as $\epsilon$ is considered to be very small. 
A nice property is that (\ref{eq:sgld}) only requires the {\em gradient} of $\log p(\theta|\lambda)$, not the log-density itself, so even though the normalizer $Z(\lambda)$ might be complicated, we can only deal with the energy function, that is, 
\begin{align}
\frac{\partial \log p(\theta|\lambda)}{\partial \theta} = - \frac{\partial E(\lambda,\theta)}{\partial \theta}. 
\end{align}

It is known that after some burn-in period, the iterates of (\ref{eq:sgld}) converge to the samples of the stationary distribution $p(\theta|\lambda)$. Also according to the Markov chain theorems~\cite{mcmc_book}, we can start from {\em any} initial iterate $\theta$ to make the chain converge to the target stationary distribution after a burn-in period.
We let $\theta^{(0)}$ be the first iterate in the SGLD recurrence: $\theta^{(0)}$ can be either independent of $\lambda$ (e.g., chosen randomly and independently of $\lambda$) or a function of $\lambda$ (e.g., $\lambda$ is the initial network parameters to be meta-learned in the MAML-type problems)\footnote{
Even though we can choose $\theta^{(0)}$ 
independently of $\lambda$, since we are only able to run the SGLD recurrences a {\em finite} number of times, it might be useful to have the first iterate $\theta^{(0)}$ dependent on $\lambda$. A good example is the MAML-type problems where we set $\theta^{(0)}\!=\!\lambda$. 
}. In the former $\frac{d \theta^{(0)}}{d \lambda}\!=\!0$, while for the latter $\frac{d \theta^{(0)}}{d \lambda}$ is non-zero, but we assume that it is {\em (sub)linearly sparse}\footnote{It means that for any $\dim(\theta)$-dim vector $v$, the product $v^\top \frac{d \theta^{(0)}}{d \lambda}$ can be computed in $O(\dim(\theta)\!+\!\dim(\lambda))$ time and space.}, and can be efficiently computed (e.g., in the MAML-type problems $\theta^{(0)}\!=\!\lambda$, and $\frac{d \theta^{(0)}}{d \lambda}\!=\!I$). 
We also let $B$ be the the number of burn-in iterations before convergence. Therefore, we run the SGLD recurrence for $(B\!+\!M)$ iterations in total, throwing away the first $B$ iterates, and collecting the last $M$ iterates. Following the notation in (\ref{eq:mc_estimate}), we set: $\theta^1\!:=\!\theta^{(B+1)}, \theta^2\!:=\!\theta^{(B+2)}, \dots, \theta^M\!:=\!\theta^{(B+M)}$.



Now, for the first recurrence $\theta^{(0)}\!\to\!\theta^{(1)}$, that is,
\begin{align}
\theta^{(1)} = \ \theta^{(0)} + \frac{\epsilon}{2} \frac{\partial \log p(\theta^{(0)}|\lambda)}{\partial \theta} + \sqrt{\epsilon} z^{(0)},
\label{eq:iter0to1}
\end{align}
by taking the gradient of both sides with respect to $\lambda$, 
\begin{align}
&\frac{d \theta^{(1)}}{d\lambda} = \ \frac{d \theta^{(0)}}{d\lambda} + \frac{\epsilon}{2} \frac{d}{d \lambda}\frac{\partial \log p(\theta^{(0)}|\lambda)}{\partial \theta} + \frac{d}{d \lambda} \sqrt{\epsilon} z^{(0)} \label{eq:dth1_dlamb_1} \\
&= \frac{d \theta^{(0)}}{d\lambda}\!+\!\frac{\epsilon}{2} \bigg( \frac{\partial^2 \log p(\theta^{(0)}|\lambda)}{\partial \lambda \partial \theta}\!+\!\frac{\partial^2 \log p(\theta^{(0)}|\lambda)}{\partial \theta^2} \cdot \frac{d \theta^{(0)}}{d\lambda} \bigg) \nonumber \\ 
& = \Big(I + \frac{\epsilon}{2} A(\lambda,\theta^{(0)})\Big) \cdot \frac{d \theta^{(0)}}{d\lambda} + \frac{\epsilon}{2} B(\lambda,\theta^{(0)}), 
\label{eq:dth1_dlamb_2}
\end{align}
where in (\ref{eq:dth1_dlamb_2}) we let 
\begin{align}
A(\lambda,\theta)\!:=\!\frac{\partial^2 \log p(\theta|\lambda)}{\partial \theta^2}, \ B(\lambda,\theta)\!:=\!\frac{\partial^2 \log p(\theta|\lambda)}{\partial \lambda \partial \theta}. 
\end{align}
We move on to the next recurrence $\theta^{(1)}\!\to\!\theta^{(2)}$ and derive its derivatives similarly. Continuing this up to $\theta^{(m)}$, we have: 
\begin{align}
\frac{d \theta^{(m)}}{d\lambda} &\simeq \bigg(I\!+\!\frac{\epsilon}{2} \Big(A(\lambda,\theta^{(0)})\!+\!\cdots\!+\!A(\lambda,\theta^{(m\!-\!1)}) \Big) \bigg) \cdot \frac{d \theta^{(0)}}{d\lambda} \nonumber \\
& \ \ \ \ \ \ + \ \frac{\epsilon}{2} \Big( B(\lambda,\theta^{(0)})\!+\!\cdots\!+\!B(\lambda,\theta^{(m\!-\!1)}) \Big). 
\label{eq:dthm_dlamb_1} 
\end{align}
The symbol $\simeq$ originates from exploitation of $\epsilon^2 \simeq 0$, and the detailed derivations for (\ref{eq:dthm_dlamb_1}) can be found in Supplement~\ref{appsec:dth_dlamb_recurr}. 
However, storing even a single $A(\lambda,\theta)$ or $B(\lambda,\theta)$ is computationally infeasible for large-scale scenarios due to its huge matrix dimension  $(\dim(\theta)\!\times\!\dim(\theta))$ or $(\dim(\theta)\!\times\!\dim(\lambda))$.

Instead, we derive a recursion directly from (\ref{eq:df_dlambda}). Specifically, we focus on the second term of (\ref{eq:df_dlambda}), denoted as $g_m(\lambda)$, 
\begin{align}
\frac{d f(\lambda, \theta^{(m)})}{d \lambda}\!=\!\frac{\partial f(\lambda, \theta^{(m)})}{\partial \lambda}\!+\!\underbrace{
\frac{\partial f(\lambda, \theta^{(m)})}{\partial \theta} \cdot \frac{d \theta^{(m)}}{d \lambda}
}_{=: g_m(\lambda)}.
\label{eq:df_dlambda_with_g}
\end{align}
Note that $g_m(\lambda)$ is a $\dim(\lambda)$-dimensional vector, thus easy to deal with computationally. Now we derive the recursion for $g_m(\lambda)$ as follows:
\begin{align}
& g_m(\lambda) = \frac{\partial f(\lambda, \theta^{(m)})}{\partial \theta} \cdot \frac{d \theta^{(m)}}{d \lambda} \label{eq:gm_recursive_1} \\
& = \frac{\partial f(\lambda, \theta^{(m)})}{\partial \theta} \cdot \Bigg(
\frac{d \theta^{(m\!-\!1)}}{d \lambda} \ + \nonumber \\
&\ \ \ \ \ \ \ \ \frac{\epsilon}{2} \bigg(
B(\lambda,\theta^{(m\!-\!1)}) + A(\lambda,\theta^{(m\!-\!1)}) \cdot \frac{d \theta^{(m\!-\!1)}}{d \lambda} \bigg)
\Bigg) \label{eq:gm_recursive_2} \\
& \simeq \frac{\partial f(\lambda, \theta^{(m)})}{\partial \theta} \cdot \Bigg(
\frac{d \theta^{(m\!-\!1)}}{d \lambda} \ + \nonumber \\
&\ \ \ \ \ \ \ \ \frac{\epsilon}{2} B(\lambda,\theta^{(m\!-\!1)}) 
+ \frac{\epsilon}{2} A(\lambda,\theta^{(m\!-\!1)}) \cdot \frac{d \theta^{(0)}}{d \lambda}
\Bigg) \label{eq:gm_recursive_3} \\
& = \underbrace{
\frac{\partial f(\lambda, \theta^{(m)})}{\partial \theta} \cdot \frac{d \theta^{(m\!-\!1)}}{d \lambda}
}_{\approx \ 
g_{m-1}(\lambda) \ + \ 
(\theta^{(m)}\!-\!\theta^{(m\!-\!1)}) \cdot \frac{\partial^2 f(\lambda, \theta^{(m\!-\!1)})}{\partial \theta^2} \cdot \frac{d \theta^{(0)}}{d \lambda}
} \ + \label{eq:gm_recursive_4}  \\
&\ \ \frac{\epsilon}{2} \frac{\partial f(\lambda, \theta^{(m)})}{\partial \theta}\!\cdot\!\Bigg( 
B(\lambda,\theta^{(m\!-\!1)}) + A(\lambda,\theta^{(m\!-\!1)}) \cdot \frac{d \theta^{(0)}}{d \lambda} \Bigg) \nonumber 
\end{align}
%
To have (\ref{eq:gm_recursive_3}) from (\ref{eq:gm_recursive_2}) we again exploit $\epsilon^2\!\simeq\!0$ using the expression (\ref{eq:dthm_dlamb_1}). 
In the under-brace of (\ref{eq:gm_recursive_4}) we use a 
first-order approximation where the details and justification can be found in Supplement~\ref{appsec:foapprox}. 
%
The quality of our first-order approximation is also empirically demonstrated in Sec.~\ref{sec:foapprox_goodness}. 

In summary, 
we have the recursion for $g_m(\lambda)$:
\begin{align}
&g_m(\lambda) \ = \ g_{m\!-\!1}(\lambda) \ + \label{eq:gm_recursive_final} \\
& \ \ \ \ (\theta^{(m)}\!-\!\theta^{(m\!-\!1)}) \cdot \frac{\partial^2 f(\lambda, \theta^{(m\!-\!1)})}{\partial \theta^2} \cdot \frac{d \theta^{(0)}}{d \lambda} \ + \nonumber \\
& \ \ \ \ \frac{\epsilon}{2} \frac{\partial f(\lambda, \theta^{(m)})}{\partial \theta}\!\cdot\!\Bigg( 
B(\lambda,\theta^{(m\!-\!1)}) + A(\lambda,\theta^{(m\!-\!1)}) \cdot \frac{d \theta^{(0)}}{d \lambda} \Bigg). \nonumber
\end{align}
In (\ref{eq:gm_recursive_final}), we can see that it is not needed to store the large matrices $A$ and $B$ at all, since we have the vector-Hessian product forms. They can be easily computed in most modern deep learning libraries with the auto-differentiation capability. In \texttt{PyTorch}, for instance, $\frac{\partial f}{\partial\theta}\!\cdot\!B$ in (\ref{eq:gm_recursive_final}) can be obtained by calling \texttt{autograd.grad()} with $t_1\!=\!\frac{\partial}{\partial \theta} \log p(\theta^{(m\!-\!1)}|\lambda)$ and $t_2\!=\!\frac{\partial}{\partial \theta} f(\lambda, \theta^{(m)})$ as:
\begin{align}
\textrm{
\texttt{grad}$\big(t_1, 
\ \lambda, \ $ \texttt{grad\_outputs}$\ = 
t_2 \big)$, 
}
\label{eq:autograd}
\end{align}
recalling that $B(\lambda,\theta) := \frac{\partial^2}{\partial \lambda \partial \theta} \log p(\theta|\lambda)$.  
We can compute $\frac{\partial f}{\partial\theta}\!\cdot\!A$ 
and  $(\theta^{(m)}\!-\!\theta^{(m\!-\!1)}) \cdot \frac{\partial^2 f(\lambda, \theta^{(m\!-\!1)})}{\partial \theta^2}$ 
in a similar manner.  
Overall the computation for (\ref{eq:gm_recursive_final}) requires only linear time and memory complexity in $\dim(\theta)$ and $\dim(\lambda)$ as we prove in Sec.~\ref{sec:complexity_analysis}. 
The initial $g_0(\lambda)$  can be easily computed\footnote{
Specifically, $g_0\!=\!0$ if $\theta^{(0)}$ is independent of $\lambda$. If $\theta^{(0)}\!=\!\lambda$ in the MAML-type initial parameter meta-learning problems, then $g_0\!=\!\frac{\partial f(\lambda, \theta^{(0)})}{\partial \theta}$, and can be easily computed.  
}. 
Once $g_m(\lambda)$ is computed, we can plug it back into (\ref{eq:df_dlambda_with_g}) and (\ref{eq:dEf_dlambda}) to compute the ultimate hypergradient. 

\noindent$\clubsuit$~\textbf{Summary.} 
Our hypergradient $h := \frac{d}{d\lambda} \mathbb{E}_{p(\theta|\lambda)}[f(\lambda,\theta)]$ can be approximately computed by the recursion: 
Initially we have $g_0(\lambda)$ and $h\!=\!0$, and for $m=1,\dots,B\!+\!M$,
\begin{align}
&\bullet \theta^{(m)}\!=\!\theta^{(m\!-\!1)} \!+\! \frac{\epsilon}{2} \frac{\partial \log p(\theta^{(m\!-\!1)}|\lambda)}{\partial \theta} \!+\! \sqrt{\epsilon} z^{(m-1)} \label{eq:hgrad_1} \\
&\bullet \textrm{Apply $g_{m\!-\!1}(\lambda) \to g_m(\lambda)$ recursion step in (\ref{eq:gm_recursive_final})} \nonumber \\
&\bullet \textrm{If $m\!>\!B$,}\ \ h \leftarrow h + \frac{1}{M} \bigg( \frac{\partial f(\lambda, \theta^{(m)})}{\partial \lambda}\!+\!g_m(\lambda) \bigg) \label{eq:hgrad_3}
%
\end{align}
Our proposed method is summarized as pseudocode in Alg.~\ref{alg:main}. Unlike the IFT methods~\cite{ift,imaml}, our approach does not require Hessian inversion (e.g., the inverse of $\frac{\partial^2}{\partial \theta^2} \log p(\theta|\lambda)$). 
Furthermore, the IFT methods are all reliant to the zero inner gradient condition, thus can be sensitive (in an unpredictable way) to the quality of the inner optimization. 

Besides, when the inner optimization can only be performed with stochastic minibatch-based updates (true for most deep learning scenarios), then this further deteriorates IFT's performance since the gradient vanishing condition would not hold for the minibatch version almost surely. 
On the other hand, in our approach, note that the gradient $\frac{\partial \log p(\theta|\lambda)}{\partial \theta}$ can be safely replaced by the minibatch stochastic gradient, as we resort to the stochastic-gradient MCMC~\cite{sgld}. This can lead to the same guarantee of the convergence to the samples from $p(\theta|\lambda)$.

\newcommand\inlineeqno{\stepcounter{equation}\ (\theequation)}
\newcommand{\INDSTATE}[1][1]{\STATE\hspace{#1\algorithmicindent}}
\begin{algorithm}[t!]
\caption{Proposed algorithm (HPO-SGLD) 
to solve (\ref{eq:stoch_optim}).
}
\label{alg:main}
\begin{footnotesize}
\begin{algorithmic}
\STATE \textbf{Input:} Problem functions $f()$ and $E()$ in (\ref{eq:stoch_optim}). \\
\ \ \ \ \ \ \ \ \ \ \ \ $B$ ($\#$ burn-in steps) and $M$ ($\#$ MC samples).
\STATE \textbf{Initialize:} $\lambda$.
\STATE For each outer iteration ($\lambda$):
    \INDSTATE[1] Choose $\theta^{(0)}$, compute $g_0(\lambda)$ and set $h=0$.
    \INDSTATE[1] For $m=1,\dots,B\!+\!M$:
        \INDSTATE[2] Perform (\ref{eq:hgrad_1}), (\ref{eq:gm_recursive_final}), and (\ref{eq:hgrad_3}) in this order. 
    \INDSTATE[1] (Now we have hypergradient $h=\frac{d}{d\lambda} \mathbb{E}_{p(\theta|\lambda)}[f(\lambda,\theta)]$)
    \INDSTATE[1] Update $\lambda \leftarrow \lambda - \eta \cdot h$ for some step size $\eta$.
\end{algorithmic}
\end{footnotesize}
\end{algorithm}

\subsection{Convergence Analysis $\&$ Complexity Analysis}\label{sec:complexity_analysis}

\textbf{Convergence Analysis.} 
If $\theta^{(0)}$ is independent of $\lambda$, we show that: i) (Theorem~\ref{thm:first_order_err}) our first-order approximation is {\em exact} in the infinitesimal sense, and from which we prove that ii) (Theorem~\ref{thm:convergence}) our HPO-SGLD algorithm converges to an optimal solution at linear convergence rate. That is, to achieve $\epsilon$-accuracy, 
we need $O(1/\epsilon)$ number of iterations. The details and proofs can be found in Supplement~\ref{appsec:convergence}. 



\begin{theorem}[Exactness of First-order Approximation] 
If (A1) and (A2) in Supplement~\ref{appsec:convergence} hold, then our first-order approximation used for $g_m(\lambda)$ recursion becomes exact (i.e., approximation error $\simeq 0$). 
\label{thm:first_order_err}
\end{theorem}

\begin{theorem}[Convergence of HPO-SGLD]
If (A1-A3) in Supplement~\ref{appsec:convergence} hold, and if the target function $f(\lambda,\theta)$ in Eq.(\ref{eq:stoch_optim}) is Lipschitz continuous in $\lambda$ for each $\theta$, then for a sufficiently large number ($M$) of the Monte Carlo samples, our HPO-SGLD algorithm (Alg.~\ref{alg:main}) converges to an optimal solution of the stochastic optimization Eq.(\ref{eq:stoch_optim}) at linear rate. More specifically, with constant step size $\eta$ that satisfies $\eta \leq \frac{1}{L}$ ($L$ is a Lipschitz constant of $f$), the following holds:
\begin{align}
&\Big| \mathbb{E}_{p(\theta|\lambda^{(k)})}\big[f(\lambda^{(k)},\theta)\big] - \mathbb{E}_{p(\theta|\lambda^*)}\big[f(\lambda^*,\theta)\big] \Big| \ \leq \nonumber \\ 
& \ \ \ \ \ \ \ \ \ \ \ \ \ \ \ \ \ \ \ \ \ \ \ \ \ \ \ \ \ \ \ \ \ \ \ \ \ \ \ \ \ \ \ \frac{||\lambda^{(0)} - \lambda^*||_2^2}{2\eta k} \ = \ O\Big(\frac{1}{k}\Big)
\end{align}
where $\lambda^{(k)}$ is the iterate at the $k$-th outer iteration, and $\lambda^*$ is an optimal solution. 
\label{thm:convergence}
\end{theorem}

\textbf{Deterministic Counterpart.} 
Our derivation in a {\em deterministic} setting, namely if we removed all stochastic components by dropping SGLD noise terms $\sqrt{\epsilon} z$ in (\ref{eq:hgrad_1}), looks very similar to  forward-mode differentiation (FMD). The main difference is that we proposed the recurrence on $g_m$ with the first-order approximation in (\ref{eq:gm_recursive_4}), 
allowing us to circumvent saving the Hessian matrices $A$ and $B$. The FMD, on the other hand, 
without such approximation, 
needs to save these matrices, making it infeasible especially for high-dim $\lambda$ cases. This is the crucial point that enables our algorithm to be practical in real-world deep learning situations. 
In Supplement~\ref{appsec:vs_fmd} we provide more detailed views on this claim, scrutinizing our deterministic version, the FMD algorithms, and the similarity between the two. 

\textbf{Computational Complexity.} 
We analyze the time and space complexity of the proposed algorithm and contrast it to FMD algorithms as well as the reverse-mode differentiation (RMD). 
The analysis is based on the techniques from algorithmic differentiation~\cite{alg_diff1,alg_diff2,fmd_rmd}. 
The complexity of the competing algorithms is summarized in Tab.~\ref{tab:complexity}, and we leave details of our analysis in Supplement~\ref{appsec:complexity}.

\begin{table}
\setlength{\tabcolsep}{3.4pt}
\vspace{-1.0em}
\caption{Time/space complexity. 
$T=$ $\#$ inner iterations. 
}
\vspace{-0.8em}
\centering
\begin{scriptsize}
\centering
\begin{tabular}{ccc}
\toprule
 & Time & Space
\\
\hline
FMD\Tstrut & $O(T \cdot \dim(\lambda) \cdot (\dim(\theta)\!+\!\dim(\lambda)))$ & $O(\dim(\theta)\!+\!\dim(\lambda))$ \\
RMD\Tstrut & $O(T\cdot(\dim(\theta)\!+\!\dim(\lambda)))$ & $O(T\cdot(\dim(\theta)\!+\!\dim(\lambda)))$ \\
Ours\Tstrut & $O(T\cdot(\dim(\theta)\!+\!\dim(\lambda)))$ & $O(\dim(\theta)\!+\!\dim(\lambda))$ \\
\bottomrule
\end{tabular}
%
\end{scriptsize}
\label{tab:complexity}
\vspace{-1.0em}
\end{table}

\section{Related Work}\label{sec:related_work}

Existing methods for hypergradient estimation in BLO can fall into two categories: {\em IFT-based} and {\em Unroll-based}. We highlight the key ideas and known issues of these approaches here, and for more comprehensive summaries we refer the readers to survey papers~\cite{blo_survey,liu2021investigatingBLO}. 
IFT hinges on the stationary condition of the inner optimal solution, and computes the hypergradient from this implicit definition. Two popular variants differ in how to approximate the Hessian inversion required in IFT -- by Neumann series approximation~\cite{ift} and  conjugate gradient ~\cite{imaml} respectively. Since both variants are reliant on the gradient-equal-to-0 stationarity condition and involve Hessian inversion, their overall performance is sensitive to the quality of the inner optimization in an unpredictable, highly nonlinear way. 

Unroll-based approaches basically approximate the inner optimal solution by an SGD iterate at some final (finite) step. So we have a chain of dependency constraints, on which we can compute the hypergradients by the chain rule. The two different ways of applying the chain rule correspond to FMD and RMD~\cite{fmd_rmd}. Despite its simplicity, RMD suffers from excessive memory usage to maintain a large computation graph built for the unrolled inner optimization, which in turn leads to truncated approximations \cite{luketina2016scalableRegularization} that bring problems of their own \cite{metz2019understanding}. 
FMD also requires large computational resources to store huge Jacobian matrices and does not scale to high-dimensional $\lambda$ \cite{micaelli2021gradient}. Although our derivation looks similar to FMD in some respects, our new approximate recurrence  circumvents FMD's inherent overhead. Other gradient-free strategies for solving BLOs such as evolution ~\cite{taylorglo} also do not scale to high-dimensional $\lambda$.

In our stochastic optimization we transform the inner loss function into a probability distribution via energy-based modeling (EBM). Previously, there were attempts to apply SGLD to EBM sampling~\cite{ebm1,ebm2}. However, there approaches have nothing to do with {\em meta learning}, and our novelty is to bring an EBM perspective to meta-learning and hyperparameter optimization.

\section{Experiments}\label{sec:expmts}

\subsection{Synthetic 1D Problem}\label{sec:expmt_synth}

We consider a simple BLO problem where in (\ref{eq:blo}) we define 
\begin{align}
\mathcal{L}_V\!(\lambda,\!\theta)\!=\!(\lambda\!-\!\theta)^2\!+\!\Big(\theta\!-\!\frac{1}{2}\Big)^2, \ 
\mathcal{L}_T(\lambda,\theta)\!=\!\frac{\theta^3}{3}\!-\!(1\!-\!\lambda^2)\theta
\nonumber
\end{align}
in $\lambda,\theta\!\in\![0,1]$. The inner loss function admits a closed-form global minimum $\theta^*(\lambda)\!=\!\sqrt{1\!-\!\lambda^2}$. And by plugging this in the outer loss function and with some 1D line search, we have the optimal solution $\lambda^*\!=\!0.7487$ at $\theta^*\!=\!0.6629$. 
We tackle this problem using the hypergradient methods, and the results summarized in Table~\ref{tab:synth1d}. For all competing methods we run up to $100$ inner iterations while in our HPO-SGLD, we split the inner iterations into $B\!=\!50$ burn-in steps and $M\!=\!50$ MC sample accumulation steps, to be fair. Other experimental details are summarized in Supplement~\ref{appsec:expmt_synth}. 

We see that both IFT-Neumann/CG and our HPO-SGLD are equally good, accurately identifying the true optimal values. On the other hand, RMD and FMD require more inner iterations as they only reach comparable errors at $1000$ inner iterations.  
We also vary the number of inner loop iterations, and the results are shown in Fig.~\ref{fig:synth1d_inner_iters}. 

\begin{table}[t!]
\setlength{\tabcolsep}{9.7pt}
\vspace{-1.0em}
\caption{Synthetic 1D problem. Test errors. 
}
\vspace{-0.8em}
\centering
\begin{scriptsize}
\centering
\begin{tabular}{ccc}
\toprule
 & $\lambda^*$ (Error) & $\theta^*$ (Error)
\\
\hline
True from line search\Tstrut & 0.7487 (-) & 0.6629 (-) \\
HPO-SGLD (Ours)\Tstrut & 0.7488 (0.0001) & 0.6629 (0.0000) \\
IFT-Neumann\Tstrut & 0.7491 (0.0004) & 0.6625 (0.0004) \\
IFT-CG\Tstrut & 0.7488 (0.0001) & 0.6628 (0.0001) \\
RMD\Tstrut & 0.7360 (0.0127) & 0.6770 (0.0141) \\
RMD-FO\Tstrut & 0.6916 (0.0572) & 0.7228 (0.0599) \\
FMD\Tstrut & 0.7360 (0.0127) & 0.6770 (0.0141) \\
\hline
RMD (1000 inner iterations)\Tstrut & 0.7489 (0.0002) & 0.6627 (0.0002) \\
RMD-FO (1000 inner iterations)\Tstrut & 0.6916 (0.0572) & 0.7226 (0.0597) \\
FMD (1000 inner iterations)\Tstrut & 0.7489 (0.0002) & 0.6627 (0.0002) \\
\bottomrule
\end{tabular}
%
\end{scriptsize}
\label{tab:synth1d}
\end{table}

\begin{figure}[t!]
\begin{center}
%
\centering
\includegraphics[trim = 2mm 3mm 3mm 3mm, clip, scale=0.28]{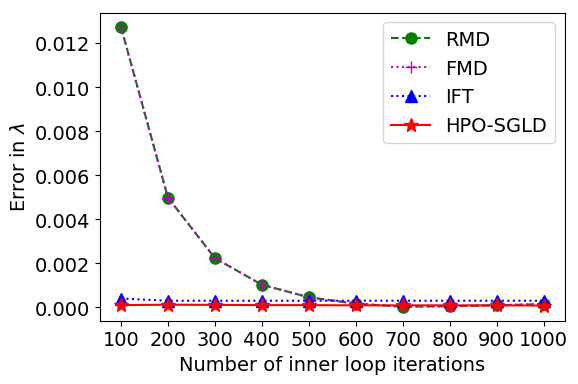} 
\includegraphics[trim = 2mm 3mm 3mm 3mm, clip, scale=0.28]{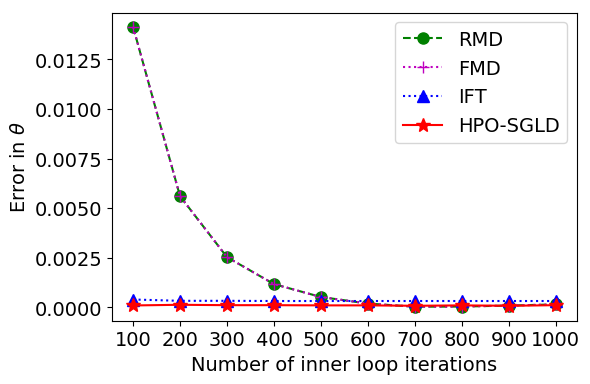} 
\end{center}
\vspace{-1.3em}
\caption{(Synthetic 1D) The number of inner iterations vs.~the errors of the learned solutions, 
(Left) $|\lambda-\lambda^*|$ and (Right) $|\theta-\theta^*|$. 
}
\label{fig:synth1d_inner_iters}
\end{figure}

\textbf{Noisy inner loss case.}
Now to make the problem more realistic and difficult, we modify the problem in a way that the inner loss function is randomly perturbed at every call. More specifically, the new inner loss function is defined as:
\begin{align}
\mathcal{L}_T(\lambda,\theta) = (1/3+\epsilon_1) \cdot \theta^3 - (1\!-\!\lambda^2+\epsilon_2)\cdot\theta
\label{eq:noisy_synth1d}
\end{align}
where $\epsilon_1,\epsilon_2\!\sim\!\textrm{Uniform}(-0.3,0.3)$. This modification makes the problem more realistic by mimicking the real-world situations where we often observe noise from various sources in the problem data (e.g., stochastic minibatch formation or noise in inputs and/or labels) in the inner loss function. 
In this case our stochastic optimization formulation (Sec.~\ref{sec:stoch_optim}) is expected to be particularly useful. The results summarized in Table~\ref{tab:noisy_synth1d} show that deterministic approaches like IFT-Neumann/CG are very sensitive to the noisy inner loss function, failing to attain the optimal values. This is mainly due to violation of the strict gradient-equal-to-0 condition for the implicit function theorem. On the other hand, our stochastic optimization treatment, for different degrees of stochasticity considered ($\tau\!=\!10^{\{-2,-5,-10\}}$), leads to more robust estimation of the optimal values. Considering the highly stochastic  nature of the problem, we see that incorporating more stochasticity (i.e., larger $\tau$) in our HPO-SGLD model, leads to more accurate estimation.


\begin{table}[t!]
\setlength{\tabcolsep}{9.7pt}
\vspace{-0.5em}
\caption{Noisy synthetic 1D problem. 
}
\vspace{-0.8em}
\centering
\begin{scriptsize}
\centering
\begin{tabular}{ccc}
\toprule
 & $\lambda^*$ (Error) & $\theta^*$ (Error)
\\
\hline
True from line search\Tstrut & 0.7487 (-) & 0.6629 (-) \\
HPO-SGLD $\tau\!=\!10^{-2}$ (Ours)\Tstrut & 0.7495 (0.0008) & 0.6516 (0.0113) \\
HPO-SGLD $\tau\!=\!10^{-5}$ (Ours)\Tstrut & 0.7473 (0.0014) & 0.6383 (0.0246) \\
HPO-SGLD $\tau\!=\!10^{-10}$ (Ours)\Tstrut & 0.7468 (0.0019) & 0.6411 (0.0218) \\
IFT-Neumann\Tstrut & 0.7640 (0.0153) & 0.6203 (0.0426) \\
IFT-CG\Tstrut & 0.7544 (0.0057) & 0.6318 (0.0311) \\
RMD\Tstrut & 0.7363 (0.0124) & 0.6779 (0.0150) \\
RMD-FO\Tstrut & 0.6923 (0.0564) & 0.7235 (0.0606) \\
FMD\Tstrut & 0.7364 (0.0123) & 0.6778 (0.0149) \\
\hline
RMD (1000 inner iterations)\Tstrut & 0.7517 (0.0030) & 0.6484 (0.0145) \\
RMD-FO (1000 inner iterations)\Tstrut & 0.6926 (0.0561) & 0.7075 (0.0446) \\
FMD (1000 inner iterations)\Tstrut & 0.7517 (0.0030) & 0.6484 (0.0145) \\
\bottomrule
\end{tabular}
%
\end{scriptsize}
\label{tab:noisy_synth1d}
\vspace{-1.0em}
\end{table}

\subsubsection{Goodness of First-Order Approximation}\label{sec:foapprox_goodness}

To demonstrate the quality of the first-order approximation that we introduced in (\ref{eq:gm_recursive_4}), that is, $\frac{\partial f(\lambda, \theta^{(m)})}{\partial \theta} \cdot \frac{d \theta^{(m\!-\!1)}}{d \lambda} \approx g_{m-1}(\lambda) + 
(\theta^{(m)}\!-\!\theta^{(m\!-\!1)}) \cdot \frac{\partial^2 f(\lambda, \theta^{(m\!-\!1)})}{\partial \theta^2} \cdot \frac{d \theta^{(0)}}{d \lambda}$, we record and compare the true product of gradient terms (left hand side) and our approximate values (right hand side). The results are visualized in Fig.~\ref{fig:synth1d_product_of_grads}, and we clearly see that the approximation quality is very good. 
We measure the errors from the beginning of our recurrences. Hence the final accumulated meta-gradient estimate would be even more accurate since we drop the meta-gradient estimates at the early iterations (relatively larger errors) as burn-in steps. 

\begin{figure}[t!]
\vspace{-1.0em}
\begin{center}
%
\centering
\includegraphics[trim = 3mm 0mm 15mm 15mm, clip, scale=0.225]{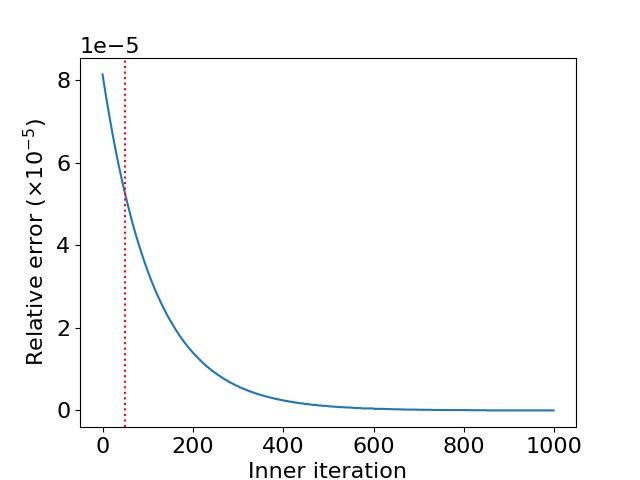} \ \ \ \ 
\includegraphics[trim = 3mm 0mm 15mm 15mm, clip, scale=0.225]{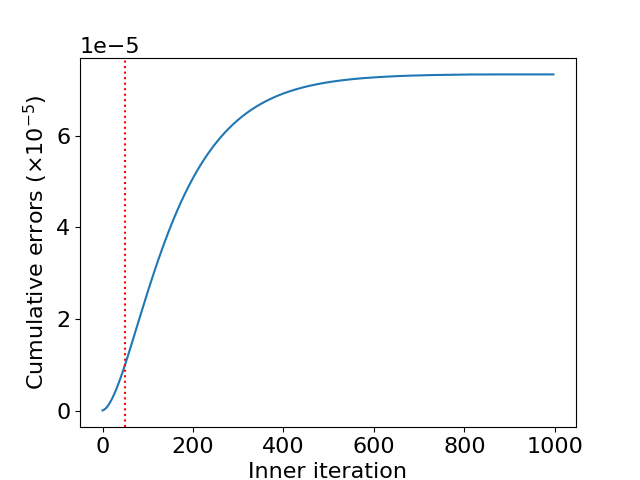}
\end{center}
\vspace{-1.5em}
\caption{
(Left) Relative errors between the true products of gradients vs.~our first-order approximates. (Right) Cumulative errors between true hypergradients and our estimates. 
Red/dotted lines indicate the end of the burn-in period. 
}
\vspace{-1.0em}
\label{fig:synth1d_product_of_grads}
\end{figure}


\subsection{L1-Regularizer HPO in ERM Learning}\label{sec:expmt_l1reg}

We test our method on the L1-regularized ERM training of a deep network as in (\ref{eq:blo_reg}). 
We consider individual weight L1 regularization, $R(\lambda,\theta)\!=\!\sum_j \lambda_j |\theta_j|$ in which $\lambda$ has the same structure as the main backbone $\theta$. 
Note that this BLO is not feasible to solve by a grid/discrete search due to the large number of hyperparameters to be searched. 
With the Vision Transformer (ViT-B-16)~\cite{vit} main network, we test the competing methods on the Oxford-Pets~\cite{pets}, DTD~\cite{dtd} and Flowers~\cite{flowers} datasets. 
The details of the experimental setups can be found in Supplement~\ref{appsec:expmt_l1reg}. 

The results are summarized in Table~\ref{tab:l1reg_hpo}. 
We see that IFT (both Neumann and CG Hessian inverse approximation) exhibits high variances for some option choices: for Neumann $(\alpha,i)$, the Hessian impact and the length of the Neumann series, respectively; for CG $(\gamma,i)$, the Hessian regularizer and the number of CG iterations, respectively. This may originate from the unknown and unstable behavior of the Hessian inverse approximation schemes especially when the gradient-equal-to-0 condition is violated. 
Our HPO-SGLD leads to relatively robust solutions for different choices of underlying options; $\tau$ and the noise scale factors ($\tau$ is the temperature parameter that turns the inner optimization problem to a probability distribution, and $\kappa$ is used for noise scaling in SGLD). That is, one of the benefits of the proposed approach is more stable final solutions, being less sensitive to underlying parameter choices, which originates from the underlying forward-mode differentiation and our stochastic treatment without relying on the gradient-equal-to-0 condition. Compared to RMD, our HPO-SGLD has significantly lower test errors. A main drawback of the RMD is its large GPU memory footprint in order to hold the full unrolled computation graph, and this considerably limits the number of inner iterations to be applied, which in turn can degrade performance. FMD is simply infeasible to run in this model scale.

\begin{table}[t!]
\setlength{\tabcolsep}{3.5pt}
\vspace{-1.0em}
\caption{L1Reg HPO results. Test errors ($\%$). 
}
\vspace{-0.8em}
\centering
\begin{scriptsize}
\centering
\begin{tabular}{ccccc}
\toprule
Methods\Tstrut & Options & Pets & DTD & Flowers
\\
\hline
 & $(0.9,1)$ & $28.07 \pm 0.65$ & $35.21 \pm 0.47$ & $15.30 \pm 0.19$ \\
 & $(0.9,3)$ & $29.47 \pm 1.13$ & $34.89 \pm 0.26$ & $15.60 \pm 0.35$ \\
IFT- & $(0.9,5)$ & $30.34 \pm 3.49$ & $35.34 \pm 0.81$ & $15.53 \pm 0.36$ \\
Neumann & $(0.9,10)$ & Diverged & Diverged & Diverged \\
$(\alpha,i)$ & $(0.99,1)$ & $28.31 \pm 0.58$ & $35.33 \pm 0.41$ & $15.43 \pm 0.31$ \\
 & $(0.99,3)$ & $28.59 \pm 1.13$ & $35.17 \pm 0.48$ & $15.81 \pm 0.49$ \\
 & $(0.99,5)$ & $28.32 \pm 0.48$ & $35.55 \pm 1.92$ & $15.62 \pm 0.35$ \\
 & $(0.99,10)$ & Diverged & Diverged & Diverged \\
\midrule
 & $(0.01,1)$ & $28.54 \pm 0.85$ & $34.80 \pm 0.31$ & $15.68 \pm 0.61$ \\
 & $(0.01,3)$ & $27.47 \pm 0.90$ & $34.91 \pm 0.17$ & $15.15 \pm 0.15$ \\
IFT- & $(0.01,5)$ & $28.22 \pm 1.90$ & $36.19 \pm 1.39$ & $15.43 \pm 0.27$ \\
CG & $(0.01,10)$ & $29.27 \pm 1.52$ & $35.05 \pm 0.13$ & $15.63 \pm 0.39$ \\
$(\gamma,i)$ & $(0.001,1)$ & $28.63 \pm 1.34$ & $34.91 \pm 0.47$ & $15.75 \pm 0.60$ \\
 & $(0.001,3)$ & $29.12 \pm 1.77$ & $35.44 \pm 0.33$ & $15.61 \pm 0.48$ \\
 & $(0.001,5)$ & $29.01 \pm 1.50$ & $35.14 \pm 0.25$ & $15.66 \pm 0.28$ \\
 & $(0.001,10)$ & $27.64 \pm 0.41$ & $35.64 \pm 0.61$ & $15.21 \pm 0.12$ \\
\midrule
\multicolumn{2}{c}{RMD}\Tstrut & $25.49 \pm 0.64$ & $36.36 \pm 0.48$ & $15.65 \pm 0.37$ \\
\multicolumn{2}{c}{RMD-FO}\Tstrut & $27.50 \pm 0.31$ & $36.26 \pm 1.06$ & $15.29 \pm 0.34$ \\
\midrule
 & $(0.01,0.01)$ & $24.80 \pm 0.32$ & $34.17 \pm 0.24$ & $15.12 \pm 0.16$ \\
 & $(0.005,0.01)$ & $23.75 \pm 0.21$ & $34.10 \pm 0.11$ & $14.86 \pm 0.10$ \\
HPO- & $(0.001,0.01)$ & $24.57 \pm 0.43$ & $33.99 \pm 0.48$ & $14.57 \pm 0.18$ \\
SGLD & $(0.0001,0.01)$ & $24.31 \pm 0.20$ & $33.86 \pm 0.15$ & $14.74 \pm 0.19$ \\
$(\tau,\kappa)$ & $(0.01,10^{-6})$ & $24.10 \pm 0.23$ & $33.60 \pm 0.21$ & $14.56 \pm 0.08$ \\
 & $(0.005,10^{-6})$ & $24.87 \pm 0.19$ & $33.53 \pm 0.46$ & $14.56 \pm 0.19$ \\
 & $(0.001,10^{-6})$ & $24.10 \pm 0.45$ & $33.97 \pm 0.28$ & $14.59 \pm 0.09$ \\
 & $(0.0001,10^{-6})$ & $24.06 \pm 0.36$ & $33.85 \pm 0.14$ & $14.67 \pm 0.30$ \\
\bottomrule
\end{tabular}
%
\end{scriptsize}
\label{tab:l1reg_hpo}
\vspace{-1.0em}
\end{table}
\subsection{Learning an Optimal Loss Function 
}\label{sec:expmt_find_loss}

Although it is conventional practice in deep learning to adopt the cross-entropy (CE) loss function for classification problems to train a deep network, it can be argued that the CE loss is not necessarily optimal, and depending on the data distributions/characteristics 
there might exist a truly optimal training loss function that is highly different from the CE loss. 
We consider the optimal loss function learning as a meta learning or HPO. 
Similarly to 
\cite{taylorglo}, we use the third-order polynomial parameterized loss function $l_\lambda(\cdot,\cdot)$ as defined in (\ref{appeq:poly_loss}). 
The loss function optimization is to find the best $\lambda\in\mathbb{R}^8$ where the model $\theta$, when trained with the loss $l_\lambda(\cdot,\cdot)$, maximizes its validation performance, can be formulated as BLO with 
\begin{align}
&\mathcal{L}_T(\lambda,\theta) := \mathbb{E}_{(x,y)\sim D_{train}}\big[ l_\lambda(f(x;\theta),y) \big], \\ 
&\mathcal{L}_V(\lambda,\theta) := \mathbb{E}_{(x,y)\sim D_{val}}\big[ CE(f(x;\theta),y) \big] 
\end{align}

The results on CIFAR-10 with Alexnet and CIFAR-100 with Wide-ResNet-28-10 (WRN)~\cite{wrn}, are summarized in Table~\ref{tab:loss_fun_optim}. 
Compared to the evolutionary search method~\cite{taylorglo} 
and RMD/FMD, our solution yields much higher test accuracy. 
FMD did not run for the large WRN network case due to its computational infeasibility. Whereas IFT methods are oftentimes sensitive to the approximation hyperparameters with large variances, our approach is less sensitive to the related options. The 
details are 
in Supplement~\ref{appsec:expmt_find_loss}. 

\begin{table}[t!]
\setlength{\tabcolsep}{8.5pt}
\vspace{-1.0em}
\caption{Loss function learning results. Test errors ($\%$). 
}
\vspace{-0.8em}
\centering
\begin{scriptsize}
\centering
\begin{tabular}{cccc}
\toprule
\multicolumn{2}{c}{Methods} & CIFAR10/Alexnet & CIFAR100/WRN 
\\
\hline
\multicolumn{2}{c}{Vanilla CE loss}\Tstrut & $23.62 \pm 0.46$ & $19.38 \pm 0.26$ \\
\midrule
\multicolumn{2}{c}{Taylor GLO}\Tstrut & $20.99 \pm 0.26$ & - \\
\midrule
 & $(0.99,1)$\Tstrut & $17.63 \pm 0.06$ & $17.36 \pm 0.10$ \\
IFT- & $(0.99,2)$\Tstrut & $17.73 \pm 0.02$ & $17.34 \pm 0.07$ \\
Neumann & $(0.99,10)$\Tstrut & $18.15 \pm 0.07$ & $17.41 \pm 0.07$ \\
$(\alpha,i)$ & $(0.9,1)$\Tstrut & $17.65 \pm 0.05$ & $17.40 \pm 0.04$ \\
 & $(0.9,2)$\Tstrut & $17.58 \pm 0.07$ & $17.38 \pm 0.06$ \\
 & $(0.9,10)$\Tstrut & $18.21 \pm 0.04$ & $17.43 \pm 0.09$ \\
\midrule
 & $(0.01,1)$ & $17.60 \pm 0.05$ & $17.39 \pm 0.06$ \\
IFT- & $(0.01,2)$ & $17.39 \pm 0.33$ & $17.41 \pm 0.04$ \\
CG & $(0.01,10)$ & $18.04 \pm 0.10$ & $17.54 \pm 0.06$ \\
$(\gamma,i)$ & $(0.001,1)$ & $17.78 \pm 0.27$ & $17.41 \pm 0.05$ \\
 & $(0.001,2)$ & $17.34 \pm 0.07$ & $17.45 \pm 0.04$ \\
 & $(0.001,10)$ & $18.96 \pm 0.12$ & $18.01 \pm 0.35$ \\
\midrule
\multicolumn{2}{c}{RMD} & $19.79 \pm 0.14$ & $18.02 \pm 0.07$ \\
\multicolumn{2}{c}{RMD-FO} & $22.74 \pm 0.15$ & $18.52 \pm 0.12$ \\
\midrule
\multicolumn{2}{c}{FMD} & $20.14 \pm 0.15$ & - \\
\midrule
\multirow{3}{*}{HPO-SGLD}\Tstrut & $\tau\!=\!0.1$ & $17.49 \pm 0.02$ & $17.24 \pm 0.05$ \\
& $\tau\!=\!0.01$ & $16.76 \pm 0.05$ & $17.20 \pm 0.03$ \\
& $\tau\!=\!0.001$ & $17.01 \pm 0.03$ & $17.24 \pm 0.04$ \\
\bottomrule
\end{tabular}
%
\end{scriptsize}
\label{tab:loss_fun_optim}
\vspace{-1.0em}
\end{table}
\subsection{Few-shot Meta Learning}\label{sec:expmt_fewshot}

We test our method on the few-shot learning (FSL) problem. In the episodic learning setup, we have many (classification) tasks of different input domains/distributions and class semantics, where the meta learner can observe each task one by one as a few labeled examples. The goal is to make a model adapt well to a novel unseen test task using only a few labeled representative samples from the test task. Despite 
abundance of existing FSL methods in the literature, in this paper we particularly focus on the MAML~\cite{maml}-type approaches where we meta-learn the initial model 
so that the learned initial model can adapt easily to a new task via a few gradient descent updates with the few-shot samples. That is, the hyperparameters $\lambda=\theta^{(0)}$. 

We compare our approach with: MAML~\cite{maml}, its first-order approximation that detaches computation graphs for high order derivative terms (FO-MAML), and the IFT-based reformulation of MAML called the iMAML~\cite{imaml}. 
We also contrast with the Reptile~\cite{reptile}, a smoothed first-order update method. 
For the backbone network the popular four-layer conv net is used, and we test the competing meta learning algorithms on the MiniImagenet dataset~\cite{matchingnet}. The results are shown in Table~\ref{tab:fsl_mini}. 
Supplement~\ref{appsec:expmt_fewshot} contains details. 
HPO-SGLD achieves consistently higher test performance than competing methods, being robust to the choice of the temperature $\tau$. 

In Supplement~\ref{appsec:expmt_fewshot}, we have additional experimental results on FSL including: a) Comparison with other Bayesian FSL methods and uncertainty quantification (Table~\ref{apptab:fsl_bayesian},~\ref{apptab:fsl_ece}), and b) Experiments on TieredImageNet (Table~\ref{apptab:fsl_tiered}).

\begin{table}
\setlength{\tabcolsep}{4.1pt}
\vspace{-1.0em}
\caption{Meta test accuracies ($\%$) on few-shot learning. 
}
\vspace{-0.8em}
\centering
\begin{scriptsize}
\centering
\begin{tabular}{cccc}
\toprule
\multicolumn{2}{c}{Methods} & 5-way, 1-shot & 5-way, 5-shot 
\\
\hline
\multicolumn{2}{c}{MAML~\cite{maml}}\Tstrut & $48.70 \pm 1.84$ & $63.11 \pm 0.92$ \\
\multicolumn{2}{c}{FO-MAML~\cite{maml}}\Tstrut & $48.07 \pm 1.75$ & $53.59 \pm 0.76$ \\
\multicolumn{2}{c}{Reptile~\cite{reptile}}\Tstrut & $49.97 \pm 0.32$ & $59.31 \pm 0.71$ \\
\multicolumn{2}{c}{iMAML~\cite{imaml}}\Tstrut & $48.96 \pm 1.84$ & $54.24 \pm 0.68$ \\
\midrule
\multirow{5}{*}{HPO-SGLD (Ours)}\Tstrut & $\tau\!=\!1$ & $50.38 \pm 0.87$ & $64.22 \pm 0.67$ \\
& $\tau\!=\!0.5$ & $50.36 \pm 0.83$ & $63.94 \pm 0.68$ \\
& $\tau\!=\!0.1$ & $50.30 \pm 0.87$ & $64.15 \pm 0.63$ \\
& $\tau\!=\!0.01$ & $50.29 \pm 0.84$ & $64.39 \pm 0.64$ \\
& $\tau\!=\!0.001$ & $50.28 \pm 0.83$ & $63.51 \pm 0.64$ \\
\bottomrule
\end{tabular}
%
\end{scriptsize}
\label{tab:fsl_mini}
\end{table}

\subsection{Meta Learning 
Implicit Neural Representation}\label{sec:expmt_meta_inr}

Next we test our method on the meta learning of the implicit neural representation (INR), following the setup similar to~\cite{meta_inr}. The main goal of the INR is to mimic a 3D imaging function $T\!:\!\mathbb{R}^3\!\to\!\mathbb{R}^4$ such that $T(x,y,z)\!=\!(r,g,b,\sigma)$ is a mapping of any 3D coordinate point of an object to its RGB color value and the depth $\sigma$. Ideally, one would like to find a neural network model $f_\theta$ (called the {\em implicit neural representation}) that is closest to the true 3D imaging function $T$, namely $f_\theta(x,y,z)\approx T(x,y,z)$ for all $(x,y,z)$. 
For a vanilla (non-meta learning) INR problem, this output matching is done with a random initial $\theta$, which often requires long training time.
The idea of meta learning of INR arose in~\cite{meta_inr}, in which we meta-learn a good initial model $\theta^0$ such that once the training starts from this $\theta^0$, it converges much more quickly and reliably than starting from a random initial. So it is similar to MAML~\cite{maml} in that the meta model we aim to learn is the {\em initial} model to start a regular training with. 
We leave details in Supplement~\ref{appsec:expmt_meta_inr}. 


We tested our approach on the task of single-view synthesis for the ShapeNet datasets~\cite{shapenet}. 
The meta learning formulation for this task was previously introduced in~\cite{meta_inr}, in which the Reptile~\cite{reptile} meta-learned initial model outperformed standard random initialization, as shown in Table~\ref{tab:meta_inr} for the single-view meta test scenarios. 
We follow the experimental protocol similar to that of~\cite{meta_inr}, and details can be found in Supplement~\ref{appsec:expmt_meta_inr}. 
Our HPO-SGLD achieves performance comparable to or sometimes better than the state-of-the-arts.

\begin{table}[t!]
\setlength{\tabcolsep}{13.5pt}
\vspace{-0.5em}
\caption{Meta learning of implicit neural representation on ShapeNet~\cite{shapenet} datasets. Test PSNR scores ($\uparrow$). 
}
\vspace{-0.8em}
\centering
\begin{scriptsize}
\centering
\begin{tabular}{cccc}
\toprule
Methods & Chairs & Cars & Lamps
\\
\hline
Standard random init\Tstrut & 12.49 & 11.45 & 15.47 \\
Reptile $+$ Output matching\Tstrut & 16.40 & 22.39 & 20.79 \\
Reptile $+$ Weight shuffling\Tstrut & 10.76 & 11.30 & 13.88 \\
Reptile~\cite{meta_inr}\Tstrut & \pmb{18.85} & 22.80 & 22.35 \\
\midrule
HPO-SGLD (Ours)\Tstrut & 18.77 & \pmb{23.16} & \pmb{22.73} \\
\bottomrule
\end{tabular}
%
\end{scriptsize}
\label{tab:meta_inr}
\vspace{-1.0em}
\end{table}


\subsection{Invariance Learning}\label{sec:expmt_augerino}

We tackle the invariance learning in neural networks~\cite{augerino}, another interesting problem that can be viewed as an HPO. The goal is to learn the invariant geometric transformations in input images for neural networks; e.g., all those translation, rotation, scaling transformations are invariant for the object classification task, while only translation and scaling are invariant for pose estimation. If we can parameterize the transformations $T\!\sim\!T_\lambda$, then the output prediction for input image $x$ can be expressed as the expectation over the allowable transformations, $\overline{g}(x,\theta,\lambda)\!:=\!\mathbb{E}_{T\!\sim\!T_\lambda}[g(T x, \theta)]$ to incorporate the invariant transformations~\cite{invar1,invar2,invar3}, where $g(x;\theta)$ is the classification scores (logits) of the neural network.
Using the Lie group algebra, one can have differentiable transformations $T_\lambda\!=\!\textrm{expm}\big( \sum_{i=1}^6 \epsilon_i \lambda_i G_i \big)$ where $\epsilon_{1:6}$ are uniform random samples from $[-1,1]$, and $G_{1:6}$ are known $(3 \times 3)$ generating matrices for X-/Y-translations, rotation, X-/Y-scaling and shearing~\cite{liegroup}.  

We can then formulate the invariance learning as BLO as detailed in 
Supplement~\ref{appsec:expmt_augerino}.
On the input image transformed MNIST datasets, 
our test results are shown in Table~\ref{apptab:auger_mnist}. 
Our approach outperforms Augerino~\cite{augerino}, and being comparable to LILA~\cite{lila} which is computationally expensive to perform full-Hessian Laplace approximation following the Bayesian model selection strategy. For the full-rotated case (Rot-180), our model returns the optimal geometric transformation hyperparameters: $\lambda^*\!=\![-0.01, 0.01, -10.0, -0.00, 0.01, -0.06]$  that correspond to X-/Y-translation, rotation, X-/Y-scale, and shear parameters, respectively. For partial rotation (Rot-90), $\lambda^*\!=\![0.02, 0.03, 1.04, 0.02, 0.06, 0.03]$. That is, the rotation parameter is correctly identified, successfully learning the input invariance that resides in the data.
For the experimental details, please refer to Supplement~\ref{appsec:expmt_augerino}.

\begin{table}[t!]
\setlength{\tabcolsep}{2.7pt}
\vspace{-1.0em}
\caption{Invariance learning. 
Test accuracies ($\%$). 
}
\vspace{-0.8em}
\centering
\begin{scriptsize}
\centering
\begin{tabular}{ccccc}
\toprule
Methods & Rot-180 & Rot-90 & Scaled & Translated
\\
\hline
Non-invariant\Tstrut & $93.82 \pm 0.10$ & $95.83 \pm 0.03$ & $97.07 \pm 0.06$ & $94.15 \pm 0.02$ \\
Augerino\Tstrut & $97.83 \pm 0.03$ & $96.35 \pm 0.02$ & $97.45 \pm 0.03$ & $94.47 \pm 0.08$ \\
LILA\Tstrut & $97.74 \pm 0.07$ & $97.81 \pm 0.11$ & $98.33 \pm 0.05$ & $97.28 \pm 0.05$ \\
\midrule
HPO-SGLD\Tstrut & $97.77 \pm 0.03$ & $97.86 \pm 0.04$ & $98.37 \pm 0.04$ & $97.30 \pm 0.04$ \\
\bottomrule
\end{tabular}
%
\end{scriptsize}
\label{apptab:auger_mnist}
\vspace{-0.5em}
\end{table}

\subsection{Ablation Study}\label{sec:ablation}

We have conducted the ablation study on the impact of the number ($M$) of the MC samples on the overall performance. On the L1-Reg HPO problem for the Pets dataset (Table~\ref{tab:l1reg_hpo} with default $M=5$), we vary the values of $M$: $M=1,3,5,7,9$ for the setting $\tau=0.001,\kappa=10^{-6}$. The test errors are shown in Table~\ref{tab:ablation}. Increasing the number of MC samples improves the test results as it takes into account multiple inner solutions and/or noisy inner losses.

\begin{table}[t!]
\setlength{\tabcolsep}{2.7pt}
\caption{Impact of the number of MC samples on test errors.
}
\vspace{-0.8em}
\centering
\begin{scriptsize}
\centering
\begin{tabular}{ccccc}
\toprule
$M=1$ & $M=2$ & $M=3$ & $M=4$ & $M=5$
\\
\hline
$25.17 \pm 0.44$\Tstrut & $24.76 \pm 0.32$ & $24.10 \pm 0.45$ & $24.07 \pm 0.39$ & $24.04 \pm 0.38$ \\
\bottomrule
\end{tabular}
%
\end{scriptsize}
\label{tab:ablation}
\vspace{-1.5em}
\end{table}

\subsection{Additional Experimental Results}\label{sec:additional_expmts}

Due to the lack of space, we have placed additional experimental results in the Supplement~\ref{appsec:additional_expmt}. They include: a) Actual wall clock running time comparison (Supp.~\ref{appsec:wall_clock}), b) Comparison with other recent stochastic methods for BLO, e.g., AmIGO approach~\cite{ref_B} (Supp.~\ref{appsec:other_stochastic_blo}), c) Synthetic quadratic function experiments (Supp.~\ref{appsec:synth_quad}), and d) Comparison with evolutionary search methods (Supp.~\ref{appsec:synth1d_es}).

\section{Conclusion}\label{sec:conclusion}

We have introduced a new hypergradient estimation method for meta learning. Our stochastic optimization formulation takes into account uncertainty in inner optimization, rendering solutions to robust to noise and non-unique inner optima. Our proposed forward recursion method enables computationally tractable solutions even in large scale scenarios (eg, 87M parameters and 87M hyperparameters for VIT-B-16).

{\small
\bibliography{main}
}

\newpage
\appendix
\onecolumn

\centerline{\huge\textbf{{Supplementary Material}}}
\vspace{+0.5em}
\centerline{\large\textbf{{A Stochastic Approach to Bi-Level optimization for Hyperparameter optimization and Meta Learning}}}

\vspace{+2.0em}

\section{Detailed Derivations}\label{appsec:derivs}

\subsection{Derivation for $\frac{d \theta^{(m)}}{d\lambda}$ in (\ref{eq:dthm_dlamb_1})}\label{appsec:dth_dlamb_recurr}

We ensure that Eq.(\ref{eq:dthm_dlamb_1}) is correct, and if $\theta^{(0)}$ does not depend on $\lambda$, then $\frac{d\theta^{(m)}}{d\lambda}$ only depends on the $B$ matrices (the second order derivatives wrt $\theta$ and $\lambda$), not on the $A$ matrices (the second order derivatives wrt $\theta$ twice). 
Recall that
\begin{align}
A(\lambda,\theta)\!:=\!\frac{\partial^2 \log p(\theta|\lambda)}{\partial \theta^2}, \ \ \ \ B(\lambda,\theta)\!:=\!\frac{\partial^2 \log p(\theta|\lambda)}{\partial \lambda \partial \theta}. 
\end{align}

This can be better illustrated by taking a look at a first few recursion steps beyond Eq.(\ref{eq:dth1_dlamb_2}). 
\begin{align}
\frac{d \theta^{(2)}}{d\lambda} &= \ \frac{d \theta^{(1)}}{d\lambda} + \frac{\epsilon}{2} \frac{d}{d \lambda}\frac{\partial \log p(\theta^{(1)}|\lambda)}{\partial \theta} + \frac{d}{d \lambda} \sqrt{\epsilon} z^{(1)} \\
& \textrm{{\em (Applying the chain rule on the second term of the RHS)}} \nonumber \\
& = 
\frac{d \theta^{(1)}}{d\lambda} \! + \! 
\frac{\epsilon}{2} \bigg(
\frac{\partial^2 \log p(\theta^{(1)}|\lambda)}{\partial \lambda \partial \theta} \! + \!
\frac{\partial^2 \log p(\theta^{(1)}|\lambda)}{\partial \theta^2} \cdot \frac{d \theta^{(1)}}{d\lambda} 
\bigg) \\ 
& = 
\Big(I + \frac{\epsilon}{2} A(\lambda,\theta^{(1)})\Big) \cdot \frac{d \theta^{(1)}}{d\lambda} + 
\frac{\epsilon}{2} B(\lambda,\theta^{(1)}) \\
& \textrm{{\em $\bigg($ Plug in Eq.(\ref{eq:dth1_dlamb_2}): $\frac{d \theta^{(1)}}{d\lambda} = \Big(I + \frac{\epsilon}{2} A(\lambda,\theta^{(0)})\Big) \cdot \frac{d \theta^{(0)}}{d\lambda} + \frac{\epsilon}{2} B(\lambda,\theta^{(0)})$ $\bigg)$}} \nonumber \\
& =
\Big(I + \frac{\epsilon}{2} A(\lambda,\theta^{(1)})\Big) \cdot \bigg( \Big(I + \frac{\epsilon}{2} A(\lambda,\theta^{(0)})\Big) \cdot \frac{d \theta^{(0)}}{d\lambda} + \frac{\epsilon}{2} B(\lambda,\theta^{(0)}) \bigg) + 
\frac{\epsilon}{2} B(\lambda,\theta^{(1)}) \\
& \textrm{{\em (Arrange terms in the product)}} \nonumber \\
& =
\bigg(I + \frac{\epsilon}{2} \Big( A(\lambda,\theta^{(0)}) + A(\lambda,\theta^{(1)}) \Big) \bigg) \cdot \frac{d \theta^{(0)}}{d\lambda} \ + \ \frac{\epsilon}{2} \Big( B(\lambda,\theta^{(0)})\!+\!
B(\lambda,\theta^{(1)}) \Big) \nonumber \\
& \ \ \ \ \ + \ \frac{\epsilon^2}{4} A(\lambda,\theta^{(1)}) \Big( A(\lambda,\theta^{(0)}) + B(\lambda,\theta^{(0)}) \Big) \\
& \textrm{{\em (Using our assumption $\epsilon^2 \simeq 0$)}} \nonumber \\
& \simeq 
\bigg(I + \frac{\epsilon}{2} \Big( A(\lambda,\theta^{(0)}) + A(\lambda,\theta^{(1)}) \Big) \bigg) \cdot \frac{d \theta^{(0)}}{d\lambda} \ + \ \frac{\epsilon}{2} \Big( B(\lambda,\theta^{(0)})\!+\!
B(\lambda,\theta^{(1)}) \Big)
\end{align}

Carrying on this way, it is not difficult to see that we can reach Eq.(\ref{eq:dthm_dlamb_1}).
And if $\theta^{(0)}$ does not depend on $\lambda$ (i.e., $\frac{d \theta^{(0)}}{d\lambda}=0$), then $\frac{d\theta^{(m)}}{d\lambda}$ only depends on the $B$ matrices (the second order derivatives wrt $\theta$ and $\lambda$), not on the $A$ matrices (the second order derivatives wrt $\theta$ twice).

\subsection{Derivation for First-Order Approximation in (\ref{eq:gm_recursive_4})}\label{appsec:foapprox}

\renewcommand{\qedsymbol}{\ensuremath{\blacksquare}}

We derive our first-order approximation introduced in (\ref{eq:gm_recursive_4}), that is, 
\begin{align}
\frac{\partial f(\lambda, \theta^{(m)})}{\partial \theta} \cdot \frac{d \theta^{(m\!-\!1)}}{d \lambda} \ \approx \ g_{m-1}(\lambda) + (\theta^{(m)}\!-\!\theta^{(m\!-\!1)}) \cdot \frac{\partial^2 f(\lambda, \theta^{(m\!-\!1)})}{\partial \theta^2} \cdot \frac{d \theta^{(0)}}{d \lambda}.
\end{align}
Recall that we defined $g_m(\lambda) := \frac{\partial f(\lambda, \theta^{(m)})}{\partial \theta} \cdot \frac{d \theta^{(m)}}{d \lambda}$.

Let $G(\theta,\theta') := \frac{\partial f(\lambda, \theta)}{\partial \theta} \cdot \frac{d \theta'}{d \lambda}$. Clearly, $g_m(\lambda) = G(\theta^{(m)}, \theta^{(m)})$, and we want to show that
\begin{align}
G(\theta^{(m)}, \theta^{(m\!-\!1)}) \approx G(\theta^{(m\!-\!1)}, \theta^{(m\!-\!1)}) + (\theta^{(m)}\!-\!\theta^{(m\!-\!1)}) \cdot \frac{\partial^2 f(\lambda, \theta^{(m\!-\!1)})}{\partial \theta^2} \cdot \frac{d \theta^{(0)}}{d \lambda}.
\label{appeq:foapprox_target}
\end{align}
To this end we will first linearize the left hand side in terms of the first argument. Due to the SGLD recurrence on $\theta^{(m)}$ as in (\ref{eq:hgrad_1}) and the small step size $\epsilon$ (esp., $\sqrt{\epsilon}$ dominates $\epsilon$), we have:
\begin{align}
\theta^{(m)} = \theta^{(m\!-\!1)} + o(\sqrt{\epsilon}).
\label{appeq:theta_diff}
\end{align}
By taking the first-order Taylor approximation for $G(\theta^{(m)}, \theta^{(m\!-\!1)})$ on the first argument at $\theta=\theta^{(m\!-\!1)}$, we get
\begin{align}
G(\theta^{(m)}, \theta^{(m\!-\!1)}) &\ \approx \ G(\theta^{(m\!-\!1)}, \theta^{(m\!-\!1)}) + o(\sqrt{\epsilon}) \cdot \frac{\partial G(\theta, \theta')}{\partial \theta}\bigg|_{\theta=\theta^{(m\!-\!1)},\theta'=\theta^{(m\!-\!1)}} \label{appeq:G_deriv_1} \\
&\ = \ G(\theta^{(m\!-\!1)}, \theta^{(m\!-\!1)}) + o(\sqrt{\epsilon}) \cdot \frac{\partial^2 f(\lambda, \theta^{(m\!-\!1)})}{\partial \theta^2} \cdot \frac{d \theta^{(m\!-\!1)}}{d \lambda}. \label{appeq:G_deriv_2}
\end{align}
Recalling from (\ref{eq:dthm_dlamb_1}), $\frac{d \theta^{(m\!-\!1)}}{d \lambda}$ can be expressed as:
\begin{align}
\frac{d \theta^{(m\!-\!1)}}{d \lambda} = \frac{d \theta^{(0)}}{d \lambda} + o(\epsilon).
\end{align}
Plugging this back into (\ref{appeq:G_deriv_2}), with $\epsilon^{1.5}\!\simeq\!0$, yields:
\begin{align}
G(\theta^{(m)}, \theta^{(m\!-\!1)}) &\ \approx \ G(\theta^{(m\!-\!1)}, \theta^{(m\!-\!1)}) + o(\sqrt{\epsilon}) \cdot \frac{\partial^2 f(\lambda, \theta^{(m\!-\!1)})}{\partial \theta^2} \cdot \frac{d \theta^{(0)}}{d \lambda}. \label{appeq:G_deriv_3}
\end{align}
Now we replace $o(\sqrt{\epsilon})$ with $\theta^{(m)}\!-\!\theta^{(m\!-\!1)}$ from (\ref{appeq:theta_diff}), and this leads us to our first-order approximation (\ref{appeq:foapprox_target}). \qed

\section{Convergence Analysis}\label{appsec:convergence}

Under the situation where the initial iterate $\theta^{(0)}$ is independent of $\lambda$, we provide convergence analysis of our algorithm and the proof for the exactness of our first-order approximation strategy.

\subsection{Theoretical Guarantee of First-Order Approximation}\label{appsec:convergence_analysis}

For convergence analysis and proving exactness of our first-order approximation, we make the following three assumptions.

\textbf{Assumptions:}
\begin{itemize}
\item (A1) $\theta^{(0)}$ independent of $\lambda$. 
\item (A2) Infinitesimal step size $\epsilon$ which makes $\epsilon^{1+\alpha} \simeq 0$ for $\alpha \geq 0.5$. 
\item (A3) SGLD iterations converge to $p(\theta|\lambda)$. 
\end{itemize}
Note that (A1) is the common assumption also made in several recent theoretical BLO papers that provided convergence proofs, while (A3)'s SGLD convergence to the stationary distribution is a widely known fact.

\begin{theorem}[Exactness of First-order Approximation] 
If (A1) and (A2) hold, then our first-order approximation used for $g_m(\lambda)$ recursion becomes exact (i.e., approximation error $\simeq 0$). That is, our hypergradient computation becomes exact.
\label{appthm:first_order_err}
\end{theorem}

\begin{proof}
We use the Taylor's theorem to analyze the first-order approximation error. Specifically in the Lagrange residual form, for any twice differentiable function $h:\mathbb{R}^n \to \mathbb{R}$, there exists a point $c\in\mathbb{R}^n$, lying on the line segment between $x$ and $a$, such that:
\begin{align}
h(x) = h(a) + \nabla h(a)^\top (x-a) + \frac{1}{2} (x-a)^\top \nabla^2 h(c) (x-a).
\label{rebut:taylor}
\end{align}
Recall from Appendix~\ref{appsec:foapprox} that we defined $G(\theta,\theta') := \frac{\partial f(\lambda, \theta)}{\partial \theta} \cdot \frac{d \theta'}{d \lambda}$, and we have $g_m(\lambda) = G(\theta^{(m)}, \theta^{(m)})$.
We apply (\ref{rebut:taylor}) to $G(\theta,\theta')$ for $\theta=\theta^{(m)}$ at around $\theta^{(m\!-\!1)}$ with $\theta'=\theta^{(m\!-\!1)}$ fixed.
\begin{align}
G(\theta^{(m)}, \theta^{(m\!-\!1)}) &\ = \ G(\theta^{(m\!-\!1)}, \theta^{(m\!-\!1)}) + (\theta^{(m)}-\theta^{(m\!-\!1)})^\top \cdot \frac{\partial G(\theta, \theta^{(m\!-\!1)})}{\partial \theta}\bigg|_{\theta=\theta^{(m\!-\!1)}} \nonumber \\
& \ \ \ \ \ \ \ + \ \frac{1}{2} (\theta^{(m)}-\theta^{(m\!-\!1)})^\top \cdot \frac{\partial^2 G(\overline{\theta}, \theta^{(m\!-\!1)})}{\partial \theta^2} \cdot (\theta^{(m)}-\theta^{(m\!-\!1)})
\end{align}
for some $\overline{\theta}$ that lies on the line segment between $\theta^{(m)}$ and $\theta^{(m\!-\!1)}$. 

Using $\frac{\partial G(\theta, \theta')}{\partial \theta} = \frac{\partial^2 f(\lambda, \theta)}{\partial \theta^2} \cdot \frac{d \theta'}{d \lambda}$ and $\frac{\partial^2 G(\theta, \theta')}{\partial \theta^2} = \frac{\partial^3 f(\lambda, \theta)}{\partial \theta^3} \cdot \frac{d \theta'}{d \lambda}$, together with the fact from Eq.(\ref{eq:dthm_dlamb_1}) where $\frac{d \theta^{(m\!-\!1)}}{d \lambda} = \frac{d \theta^{(0)}}{d \lambda} + O(\epsilon)$, and the SGLD recurrence $\theta^{(m)} - \theta^{(m\!-\!1)} = O(\sqrt{\epsilon})$, we have
\begin{align}
G(\theta^{(m)}, \theta^{(m\!-\!1)}) &\ = \ G(\theta^{(m\!-\!1)}, \theta^{(m\!-\!1)}) + O(\sqrt{\epsilon}) \cdot \frac{\partial^2 f(\lambda, \theta^{(m\!-\!1)})}{\partial \theta^2} \cdot \bigg( \frac{d \theta^{(0)}}{d \lambda} + O(\epsilon) \bigg) \nonumber \\
& \ \ \ \ \ \ \ + \ \frac{1}{2} O(\sqrt{\epsilon}) \cdot \frac{\partial^3 f(\lambda, \overline{\theta})}{\partial \theta^3} \cdot \bigg( \frac{d \theta^{(0)}}{d \lambda} + O(\epsilon) \bigg) \cdot O(\sqrt{\epsilon})
\end{align}
Arranging terms while removing those that vanish due to (A2) $\epsilon^{1+\alpha} \simeq 0$ for $\alpha\geq 0.5$, yields:
\begin{align}
G(\theta^{(m)}, \theta^{(m\!-\!1)}) &\ \simeq \ G(\theta^{(m\!-\!1)}, \theta^{(m\!-\!1)}) + O(\sqrt{\epsilon}) \cdot \frac{\partial^2 f(\lambda, \theta^{(m\!-\!1)})}{\partial \theta^2} \cdot \frac{d \theta^{(0)}}{d \lambda} + \frac{1}{2} O(\epsilon) \cdot \frac{\partial^3 f(\lambda, \overline{\theta})}{\partial \theta^3} \cdot \frac{d \theta^{(0)}}{d \lambda}
\end{align}
At last, if (A1) $\frac{d \theta^{(0)}}{d \lambda} = 0$, we get $G(\theta^{(m)}, \theta^{(m\!-\!1)}) \simeq G(\theta^{(m\!-\!1)}, \theta^{(m\!-\!1)})$, which implies that our recurrrence on $g_m(\lambda)$ is exact (in the infinitesimal $\simeq$ sense). 
\end{proof}

\begin{theorem}[Convergence of HPO-SGLD]
If (A1-A3) hold, and if the target function $f(\lambda,\theta)$ in Eq.(\ref{eq:stoch_optim}) is Lipschitz continuous in $\lambda$ for each $\theta$, then for a sufficiently large number ($M$) of the Monte Carlo samples, our HPO-SGLD algorithm (Alg.~\ref{alg:main}) converges to an optimal solution of the stochastic optimization Eq.(\ref{eq:stoch_optim}) at linear rate. More specifically, with constant step size $\eta$ that satisfies $\eta \leq \frac{1}{L}$ ($L$ is a Lipschitz constant of $f$), the following holds:
\begin{align}
\Big| \mathbb{E}_{p(\theta|\lambda^{(k)})}\big[f(\lambda^{(k)},\theta)\big] - \mathbb{E}_{p(\theta|\lambda^*)}\big[f(\lambda^*,\theta)\big] \Big| \ \leq \ \frac{||\lambda^{(0)} - \lambda^*||_2^2}{2\eta k} \ = \ O\Big(\frac{1}{k}\Big)
\end{align}
where $\lambda^{(k)}$ is the iterate at the $k$-th outer iteration, and $\lambda^*$ is an optimal solution. 
\end{theorem}
\begin{proof}
This linear convergence is an immediate consequence of Theorem~\ref{thm:first_order_err} on the exactness of the hypergradient estimation in case of $\frac{d \theta^{(0)}}{d\lambda}=0$, in which the fixed step size gradient descent converges at linear rate. 
The theorem follows Theorem 2.1 of~\cite{ghadimi_lan}, which relies on the assumptions (in our context): i) $E[f(\cdot,\theta)]$ is bounded below, ii) $\nabla E[f(\cdot,\theta)]$ is $L$-Lipschitz continuous, and iii) the variance of the gradient $\nabla E[f(\cdot,\theta)]$ is bounded. We do not require $E[f(\cdot,\theta)]$ to be convex. 
\end{proof}

\section{Complexity Analysis}\label{appsec:complexity_analysis}

Our derivation in a {\em deterministic} version, namely if we removed all stochastic components by dropping all SGLD noise terms $\sqrt{\epsilon} z$ in (\ref{eq:hgrad_1}), looks very similar to the forward-mode differentiation (FMD). The main difference is that we proposed the recurrence on $g_m$ with the first-order approximation in (\ref{eq:gm_recursive_4}), 
allowing us to circumvent saving the Hessian matrices $A$ and $B$. The FMD, on the other hand, 
without such approximation, 
needs to save these matrices, leading to infeasibility especially for high-dim $\lambda$ cases. This is the crucial point that enables our algorithm practical in real-world deep learning situations. This section provides more detailed views on this claim, and analyze computational complexity of the proposed algorithm. First, we detail our deterministic version, the FMD algorithms and the similarity between the two.

\subsection{Deterministic Version of Ours vs.~FMD}\label{appsec:vs_fmd}

The FMD is typically described in the {\em optimizer learning} context~\cite{fmd_rmd}. 
Formally, let $\theta^t\!=\!\Phi_t(\theta^{t-1}, \lambda)$ be the optimizer update step at (inner) iteration $t$, and our goal is to find $\lambda$ that minimizes the outer loss $\mathcal{L_V}(\lambda,\theta^T)$ at the final iterate $\theta^T$. For instance, the SGD step takes $\Phi_t(\theta^{t-1}, \lambda)\!=\!\theta^{t-1}\!-\!\gamma \frac{\partial \mathcal{L}_T(\lambda,\theta^{t-1})}{\partial\theta}$, with the inner loss $\mathcal{L}_T(\lambda,\theta)$. The problem then becomes:
\begin{align}
\min_\lambda \mathcal{L_V}(\lambda, \theta^T) \ \textrm{s.t.} 
\ \big\{ \theta^t\!=\!\Phi_t(\theta^{t-1}, \lambda) \big\}_{t=1}^T 
\label{eq:fmd_obj}
\end{align}
For each of the $T$ constraints, by taking the gradient with respect $\lambda$, we have the recursion:
\begin{align}
\underbrace{\frac{d \theta^t}{d \lambda}}_{=:s_t} = \underbrace{\frac{\partial \Phi_t(\theta^{t-1}, \lambda)}{\partial \theta}}_{=:a_{t-1}} \cdot \underbrace{\frac{d \theta^{t-1}}{d \lambda}}_{=s_{t-1}} + \underbrace{\frac{\partial \Phi_t(\theta^{t-1}, \lambda)}{\partial \lambda}}_{=:b_{t-1}},
\label{eq:fmd_recursion}
\end{align}
and the hypergradient becomes:
\begin{align}
\frac{d \mathcal{L_V}(\lambda, \theta^T)}{d \lambda} = \frac{\partial \mathcal{L_V}(\lambda, \theta^T)}{\partial \lambda} + \frac{\partial \mathcal{L_V}(\lambda, \theta^T)}{\partial \theta} \cdot s_T
\label{eq:fmd_hypergrad}
\end{align}
That is, to compute the hypergradient, we need $s_T$; and to this end, we have to keep track of $s_t$ (i.e., store/update $s_t$ in memory)  from $s_t\!=\!a_{t-1}\!\cdot\!s_{t-1}\!+\!b_{t-1}$. 
Note that $s_t$ is $(\dim(\theta) \times \dim(\lambda))$, being infeasible to store in memory for large $\dim(\lambda)$, e.g., if $\lambda$ is of the same structure as $\theta$.

There is a way to prevent this infeasible memory storing for FMD, but the idea is essentially to apply the forward recursion for {\em individual} (scalar) $\lambda_j$ for each $j=1\dots\dim(\lambda)$. More specifically, the $j$-th component of the hypergradient, $\frac{d \mathcal{L_V}(\lambda, \theta^T)}{d \lambda_j} = \frac{\partial \mathcal{L_V}(\lambda, \theta^T)}{\partial \lambda_j} + \frac{\partial \mathcal{L_V}(\lambda, \theta^T)}{\partial \theta} \cdot s_T[:,j]$, and we can only store/update the $j$-th column of $s_t$ in the recursion. Although this reduces the memory complexity to $O(\dim(\theta))$, it incurs infeasible time complexity since we have to run the whole $T$-step recursions $\dim(\lambda)$ times.
That is, for high-dim $\lambda$, FMD is either impractical in memory or impractical in time.  

Now we write down a deterministic version of our algorithm,  show its similarity to FMD, and highlight how our method can be practical thanks to our proposed approximations. 
First, we set the number of MC samples for average, $M\!=\!1$ and let $T\!=\!B\!+\!1$. 
The SGLD recursion step (\ref{eq:hgrad_1}), after dropping the noise term and incorporating the training-loss energy form (\ref{eq:energy_train_loss}), boils down to:
\begin{align}
\theta^{(m)} = \theta^{(m-1)} - \Big(\frac{\epsilon}{2\tau}\Big) \frac{\partial \mathcal{L}_T(\lambda,\theta^{(m-1)})}{\partial\theta},
\label{eq:sgld_deterministic}
\end{align}
which is equivalent to the SGD update step $\Phi_t(\theta^{t-1}, \lambda)\!=\!\theta^{t-1}\!-\!\gamma \frac{\partial \mathcal{L}_T(\lambda,\theta^{t-1})}{\partial\theta}$ in FMD, by letting $\gamma = \epsilon/(2\tau)$. If we take the derivative on this recursion with respect to $\lambda$, clearly we have correspondence $a_t \equiv I + \gamma A_m$ and $b_t \equiv \gamma B_m$. However, instead of running the recursion (\ref{eq:fmd_recursion}) directly (which would require huge Jacobian matrix storing), we introduced $g_m(\lambda)=\frac{\partial f(\lambda, \theta^{(m)})}{\partial \theta} \cdot \frac{d \theta^{(m)}}{d \lambda}$. With our first-order approximation, we were able to derive recursion for $g_m$ as (\ref{eq:gm_recursive_final}), which only involves storing $g_m(\lambda)$ that takes only $O(\dim(\lambda))$ memory space. As will be shown in the next section, running (\ref{eq:gm_recursive_final}) also takes only linear time $O(\dim(\theta)+\dim(\lambda))$, thus practical for real-world deep learning.

\subsection{Computational Complexity}\label{appsec:complexity}

We analyze the time and space complexity of the proposed algorithm and contrast it to FMD algorithms as well as the reverse-mode differentiation (RMD). 
The analysis used in this section is based on the techniques from algorithmic differentiation~\cite{alg_diff1,alg_diff2,fmd_rmd}. Specifically, we use the following theorem/fact:
\begin{theorem}[\cite{alg_diff1}]
Let $F:\mathbb{R}^n\to\mathbb{R}^p$ be a differentiable function where its evaluation $Y\!=\!F(X)$ takes $c(n,p)$ time and $s(n,p)$ space. (i) For any vector $v\in\mathbb{R}^n$, we can compute $\frac{\partial F}{\partial X} \cdot v$ by forward-mode differentiation in $O(c(n,p))$ time and $O(s(n,p))$ space. 
(ii) For any vector $u\!\in\!\mathbb{R}^p$, we can compute $u^\top\!\cdot\!\frac{\partial F}{\partial X}$ by reverse-mode differentiation in $O(c(n,p))$ time and $O(c(n,p))$ space. 
\end{theorem}

First, we apply the theorem to the FMD algorithm, especially the column-wise storing version (the matrix version being even worse). We assume that the inner loss $\mathcal{L}_T(\lambda,\theta)$ can be computed in $O(\dim(\theta)\!+\!\dim(\lambda))$ both time and space. Regarding the inner loss as a scalar function $Y\!=\!F(X)$ with $n\!=\!\dim(\theta)\!+\!\dim(\lambda)$ and $p\!=\!1$, the theorem says that the inner optimization SGD step $\theta^t\!=\!\Phi_t(\theta^{t-1}, \lambda)$ takes $O(\dim(\theta)\!+\!\dim(\lambda))$ time and space. 
Next, regarding $\theta^t\!=\!\Phi_t(\theta^{t-1}, \lambda)$ as $Y\!=\!F(X)$ with $n\!=\!\dim(\theta)\!+\!\dim(\lambda)$ and $p\!=\!\dim(\theta)$, the column-wise recursion $s_t[:,j]\!=\!a_{t-1}\!\cdot\!s_{t-1}[:,j]\!+\!b_{t-1}\cdot e_j$ (where $e_j$ is the all-0 vector but 1 at index $j$) can be computed in $O(\dim(\theta)\!+\!\dim(\lambda))$ both time and space. 
Since we run the recursion for $T$ steps for $j\!=\!1,\dots,\dim(\lambda)$, the time complexity is $O(T\cdot \dim(\lambda) \cdot (\dim(\theta)\!+\!\dim(\lambda)))$ while the space complexity remains at $O(\dim(\theta)\!+\!\dim(\lambda))$. The product form in time complexity makes it infeasible. 

Next, the RMD just maintains the whole computation graph from $\theta^0$ to $\theta^T$ and the outer loss. Thus the complexity for computing a hypergradient is $O(T\cdot(\dim(\theta)\!+\!\dim(\lambda)))$ in both time and space. The space complexity that grows linearly with the number of inner optimization steps, can be problematic in practice oftentimes. 

At last, our algorithm can be analyzed in a similar manner using the theorem. First, the SGLD step (\ref{eq:hgrad_1}) takes $O((\dim(\theta)\!+\!\dim(\lambda)))$ time and space (Note that precisely two times the space is needed since we have to maintain both $\theta^m$ and $\theta^{(m-1)}$ for the $g_m$ recursion). For the $g_m$ recursion (\ref{eq:gm_recursive_final}), we apply the theorem with $F = \frac{\partial\log p(\theta|\lambda)}{\partial\theta}$, which requires only $O((\dim(\theta)\!+\!\dim(\lambda)))$ time and space. Lastly, the hypergradient accumulation step (\ref{eq:hgrad_3}) takes $O(\dim(\theta)\!+\!\dim(\lambda))$ time and space. These three steps go over $T\!=\!B\!+\!M$ times, but with the space overwritten. Thus the complexity for computing a hypergradient takes $O(T \cdot (\dim(\theta)\!+\!\dim(\lambda)))$ time and $O(\dim(\theta)\!+\!\dim(\lambda))$ space. Note that the space does not grow with the number of inner steps, thus clearly beneficial in increasing the number of inner steps compared to RMD.

\section{Details of Experimental Designs and Setups }\label{appsec:expmt}

\subsection{Synthetic 1D BLO Problem}\label{appsec:expmt_synth}

The experimental details for Sec.~\ref{sec:expmt_synth}. 
For the competing hypergradient methods, the running options are chosen fairly: 200 outer iterations (learning rate 0.005) and 100 inner iterations (learning rate 0.005). In our HPO-SGLD, we split the inner iterations into 50 burn-in steps and $M=50$ MC sample accumulation steps. The temperature $\tau=10^{-6}$ is chosen considering deterministic nature of the given BLO problem. The noise scale $\kappa=1$. For IFT-Neumann~\cite{ift}, we set Hessian discounting factor $\alpha=0.99$ and Neumann series length $i=10$. IFT-CG~\cite{imaml} is the IFT strategy using the conjugate gradient matrix inversion, 
and we choose the Hessian regularizer $\gamma=0.01$ and the number of CG iterations 10. RMD-FO stands for RMD with first-order approximation.

\subsection{L1-Regularizer HPO in ERM Learning}\label{appsec:expmt_l1reg}

We take random subsets as training and validations sets. Considering the class cardinality and level of difficulty in prediction across datasets, we set: (Pets) $3\%$ random subset for each training/validation set, (DTD) $40\%$ and (Flowers) $20\%$. Before running the HPO algorithms, for all competing methods we have a warm-up training stage with the feature extractor module initialized by pre-trained weights. The warm-up training is early-stopped using the validation set we have. 

Our HPO-SGLD uses $(B,M)\!=\!(5,5)$ for burn-in steps and the number of MC samples with batch size 8. This matches the number of inner steps 10 used in other methods except for RMD which only allows up to 3 inner iterations on a Tesla V100 GPU. FMD is simply infeasible to run. The inner loop learning rate is $0.01$ for all methods, and the outer loop learning rate is $0.01$ for our HPO-SGLD and $0.5$ for the other methods, selected by cross validation. The outer loop is iterated up to 1000 iterations.

\subsection{Learning an Optimal Loss Function 
}\label{appsec:expmt_find_loss}

It is conventional practice in deep learning to adopt the cross-entropy (CE) loss function for classification problems to train a deep network. However, it can be argued that the CE loss is not necessarily optimal, and depending on the particular training dataset that we work on, there might exist a truly optimal training loss function that is highly different from the CE loss. In other words, one can think of the {\em loss function learning problem}, which is personalized for a particular problem domain. We consider the optimal loss function learning as a meta learning or HPO, and aim to tackle it using our proposed method.

Similarly as 
\cite{taylorglo} 
we use the third-order polynomial parameterized loss function as follows. Let $y$ be $C$-dim one-hot true label and $f(x;\theta)\in\mathbb{R}^C$ be model's prediction logit or class probability vector with the model's parameters $\theta$ (e.g., network weights). 
\begin{align}
&l_\lambda(f(x;\theta),y) = -\frac{1}{C} \sum_{j=1}^C \bigg( 
\lambda_2 (f(x;\theta)_j\!-\!\lambda_1) + \frac{\lambda_3}{2} (f(x;\theta)_j\!-\!\lambda_1)^2 + \frac{\lambda_4}{6} (f(x;\theta)_j\!-\!\lambda_1)^3 \ + 
\label{appeq:poly_loss} \\
& \ \ \ \ \ \ \ \ \lambda_5 (f(x;\theta)_j\!-\!\lambda_1) (y_j\!-\!\lambda_0) + \frac{\lambda_6}{2} (f(x;\theta)_j\!-\!\lambda_1)^2 (y_j\!-\!\lambda_0) + \frac{\lambda_7}{2} (f(x;\theta)_j\!-\!\lambda_1) (y_j\!-\!\lambda_0)^2 \bigg), \nonumber
\end{align}
where $\lambda\in\mathbb{R}^8$ is the parameters of the loss function (i.e., hyperparameters) to be learned.

The loss function optimization is to find the best $\lambda$ where the model $\theta$, when trained with the loss $l_\lambda(\cdot,\cdot)$, maximizes its validation performance. This can be formulated as a BLO problem,
\begin{align}
\min_\lambda \ \mathcal{L}_V(\lambda,\theta^*(\lambda)) \ \ \textrm{s.t.} \ \ \theta^*(\lambda) = \arg\min_\theta \mathcal{L}_T(\lambda,\theta),
\label{eq:blo_find_loss}
\end{align}
where we define
\begin{align}
&\mathcal{L}_T(\lambda,\theta) := \mathbb{E}_{(x,y)\sim D_{train}}\big[ l_\lambda(f(x;\theta),y) \big], \\ 
&\mathcal{L}_V(\lambda,\theta) := \mathbb{E}_{(x,y)\sim D_{val}}\big[ CE(f(x;\theta),y) \big] 
\end{align}

The results on CIFAR-10 with Alexnet are summarized in Table~\ref{tab:loss_fun_optim}. 
Compared to the evolutionary search method~\cite{taylorglo}, denoted by Taylor GLO, our solution yields much higher test accuracy.
For IFT and our HPO-SGLD, we use 10 inner iterations, $50K$ outer iterations. HPO-SGLD uses $\kappa=10^{-6}$ and $(B,M)=(8,2)$.
We also ran experiments with a larger dataset and network model. Table~\ref{tab:loss_fun_optim} shows results on CIFAR-100 with Wide-ResNet-28-10 (WRN)~\cite{wrn}. We split the data into $90\%/10\%$ training/validation sets. 
The batch size is 32 for CIFAR-10 and 128 for CIFAR-100 experiments, and the number of outer iterations is $50K$ for both datasets. The outer and inner loop learning rates are $0.01$ for both datasets and for all methods.

\subsection{Few-shot Learning}\label{appsec:expmt_fewshot}

We follow the well-known protocols for the experiments with the miniImagenet dataset~\cite{maml,imaml,reptile}. For the competing methods, the number of inner-loop iterations is chosen as 10. MAML uses inner and outer learning rates $0.01$ and $0.002$, respectively;  iMAML has inner and outer learning rates $0.5$ and $0.002$, respectively, with 5 CG steps and the regularization coefficient $0.5$; Reptile uses inner and outer learning rates $0.01$ and $10^{-5}$, respectively. 
For our HPO-SGLD, we use $(B,M)=(4,1)$ setting, 
the inner and outer loop learning rates are 0.5 and 0.0005, respectively, and we fix $\kappa=10^{-4}$.

\subsubsection{Comparison with Bayesian FSL Methods and Uncertainty Quantification}\label{appsec:fsl_bayesian}

Although our approach is not limited to the few-shot learning (FSL) problem but general enough to be applicable to any BLO problems, there are some specially tailored Bayesian FSL algorithms (e.g., Bayesian MAML~\cite{bmaml} and Amortized Bayesian Meta Learning~\cite{abml}), and first we compare our HPO-SGLD algorithm with these approaches on the miniImageNet dataset. The results in  Table~\ref{apptab:fsl_bayesian} show that our HPO-SGLD performs similarly to~\cite{bmaml} and better than~\cite{abml}. 
We also compare the uncertainty quantification measure, especially the ECE score~\cite{calibration}. The results in Table~\ref{apptab:fsl_ece} show that our HPO-SGLD has comparable performance to Bayesian MAML in terms of uncertainty quantification 
even though our method is a very general framework not limited to FSL. 

\begin{table}[t!]
\setlength{\tabcolsep}{4.7pt}
\caption{Comparison to existing Bayesian FSL methods on miniImageNet. Test accuracy.
}
\vspace{-0.8em}
\centering
\begin{scriptsize}
\centering
\begin{tabular}{ccc}
\toprule
& 5-way 1-shot & 5-way 5-shot 
\\
\hline
BMAML~\cite{bmaml}\Tstrut & $49.17 \pm 0.87$ & $64.23 \pm 0.69$ \\
ABML~\cite{abml}\Tstrut & $45.00 \pm 0.60$ & -- \\
HPO-SGLD (Ours)\Tstrut & $50.38 \pm 0.87$ & $64.22 \pm 0.67$ \\
\bottomrule
\end{tabular}
%
\end{scriptsize}
\label{apptab:fsl_bayesian}
\end{table}

\begin{table}[t!]
\setlength{\tabcolsep}{4.7pt}
\caption{ECE uncertainty quantification scores on the miniImageNet 5-way 1-shot case. The smaller the better.
}
\vspace{-0.8em}
\centering
\begin{scriptsize}
\centering
\begin{tabular}{cc}
\toprule
Methods & ECE
\\
\hline
MAML~\cite{maml}\Tstrut & $3.80\%$ \\
BMAML~\cite{bmaml} & $2.40\%$ \\
HPO-SGLD (Ours) & $2.35\%$ \\
\bottomrule
\end{tabular}
%
\end{scriptsize}
\label{apptab:fsl_ece}
\end{table}

\subsubsection{Few-shot Learning Experiments on TieredImageNet}\label{appsec:fsl_tiered}

We also report the performance of our HPO-SGLD on the tieredImageNet dataset. As shown in Table~\ref{apptab:fsl_tiered} our approach attains favourable results compared to MAML and ProtoNet. 

\begin{table}[t!]
\setlength{\tabcolsep}{4.7pt}
\caption{Few-shot learning results on the tieredImageNet dataset. Test accuracy.
}
\vspace{-0.8em}
\centering
\begin{scriptsize}
\centering
\begin{tabular}{ccc}
\toprule
& 1-shot & 5-shot
\\
\hline
MAML\Tstrut & $51.67 \pm 1.81$ & $70.30 \pm 1.75$ \\
ProtoNet & $53.31 \pm 0.89$ & $72.69 \pm 0.74$ \\
HPO-SGLD (Ours) & $55.91 \pm 0.72$ & $72.94 \pm 0.47$ \\
\bottomrule
\end{tabular}
%
\end{scriptsize}
\label{apptab:fsl_tiered}
\vspace{-1.5em}
\end{table}

\subsection{Meta Learning of Implicit Neural Representation}\label{appsec:expmt_meta_inr}

We test our method on the meta learning of the implicit neural representation (INR). The main goal of the INR is to mimic a 3D imaging function $T:\mathbb{R}^3\to\mathbb{R}^4$ such that $T(x,y,z)=(r,g,b,\sigma)$ is a mapping of any 3D coordinate point of an object to its RGB color value and the depth $\sigma$. Ideally, one would like to find a neural network model $f_\theta$ (called the {\em implicit neural representation}) that is closest to the true 3D imaging function $T$, namely $f_\theta(x,y,z)\approx T(x,y,z)$ for all $(x,y,z)$. In practice, however, the true 3D function $T$ is usually not available, but only some rendered 2D images (i.e., pixel values) at certain camera poses or views. We let $\mathcal{M}$ be a (known) volumetric rendering function, more specifically, $\mathcal{M}(T,p)$ returns the rendered 2D image at the camera view $p$. 

For a vanilla (non-meta learning) INR problem, we are given a set of several pairs of view and the corresponding rendered pixels for a given object/scene; say, the training set can be formed as $\mathcal{P}=\{(p_i,D_i)\}_{i=1}^m$ of $m$ views where $D_i=\{(r_j,g_j,b_j,\sigma_j)=\mathcal{M}(T,p_i)_j\}_{j=1}^{n_i}$ constitutes the rendered pixel values for $n_i$ points at the camera view $p_i$. The training of INR amounts to minimizing $L(\theta,T,\mathcal{P}) := \mathbb{E}_{(p,D)\sim \mathcal{P}} \mathbb{E}_{j\sim D} ||\mathcal{M}(f_\theta,p)_j-\mathcal{M}(T,p)_j||^2$. The test is then performed on some novel views $p^*$ of the same object/scene, the performance being measured as the the L2 distance (or its inverse $\log$ known as the PSNR) between the rendered color values of the learned INR and the true values.

But the main drawback of this vanilla training is that  for each training object/scene, the training has to be restarted with the random initial model $\theta^0$, even though we have a large number of objects/scenes that are relevant to each other (e.g., lamp objects with different types and styles). The idea of meta learning of INR arose in~\cite{meta_inr}, in which we aim to meta-learn a good initial model $\theta^0$ such that once the training starts from this $\theta^0$, it converges much more quickly and reliably than starting from a random initial. 
So it is similar to MAML~\cite{maml} in that the meta model we aim to learn is the {\em initial} model to start a regular training with.

To this end, we define a meta training data set $\mathcal{D}^{tr}=\{(T^{tr}_t, \mathcal{P}^{tr}_t)\}_t$ of different objects/scenes $T^{tr}_t$ and corresponding views and rendered color values $\mathcal{P}^{tr}_t$. Similarly, the validation data set $\mathcal{D}^{va}$ is formed for non-overlapping views. We can then formulate the meta INR learning problem as a BLO in (\ref{eq:blo}) 
where the outer objective is:
\begin{align}
\mathcal{L}_V(\lambda,\theta^*(\lambda)) = \mathbb{E}_{(T,\mathcal{P})\sim \mathcal{D}^{va}} [L(\theta^*(\lambda), T, \mathcal{P})],
\end{align}
and the inner optimum $\theta^*(\lambda) = \theta^K$ as we use the unrolled form to specify dependency on the initial $\lambda = \theta^0$ explicitly. That is, for $k=1,\dots,K$, 
\begin{align}
\theta^k\!=\!\theta^{k-1} - \eta \nabla_\theta \mathbb{E}_{(T,\mathcal{P})\sim \mathcal{D}^{tr}} [L(\theta^{k-1}, T, \mathcal{P})].
\label{eq:meta_inr_inner}
\end{align}

We tested our approach on the task of single-view synthesis for the ShapeNet datasets~\cite{shapenet}. 
The meta learning formulation for this task was previously introduced in~\cite{meta_inr}, and we followed the experiment protocols from it: the simplified NeRF~\cite{simplified_nerf} is used, a single MLP with 6 layers, each with 256 channels and ReLU activations. The inner and outer learning rates are 0.5 and $5\times 10^{-5}$ ($5\times 10^{-4}$ for the Chairs dataset), and the outer loops run for $45$ (Lamps), $30$ (Cars) and $15$ epochs (Chairs). There are 25 train views and 1 test view, and the batch size is 128. 
In~\cite{meta_inr}, they showed that the Reptile~\cite{reptile} meta-learned initial model outperformed standard random initialization, as shown in Table~\ref{tab:meta_inr} for the single-view meta test scenarios. The Table also shows that some variants of the Reptile-optimized initial network, either training another network to match the outputs or randomly shuffling the network weights, are inferior to the original meta-trained one.  
We follow the experimental setup identical to that of~\cite{meta_inr}, and run our HPO-SGLD to this task. 
Motivated from Reptile's progressive meta-gradient update scheme, namely $\lambda$ is updated along the direction of the inner optima $\theta^*(\lambda)$, we add the regularization term $||\lambda\!-\!\textrm{sg}(\theta^*(\lambda))||^2$ to our validation loss, where  $\textrm{sg}$ refers to stop-gradient. 
With $\tau\!=\!10^{-4}$, $\kappa=10^{-6}$ and $(B,M)\!=\!(27,5)$, where the latter matches the 32 inner steps for Reptile, our HPO-SGLD achieves performance comparable to or sometimes better than the state-of-the-arts.

\subsection{Invariance Learning}\label{appsec:expmt_augerino}

We tackle the invariance learning in neural networks~\cite{augerino}, another interesting problem that can be viewed as an HPO. The goal is to learn the invariant geometric transformations in input images for neural networks; e.g., all those translation, rotation, scaling transformations are invariant for the object classification task, while only translation and scaling are invariant for pose estimation. If we can parameterize the transformations $T\!\sim\!T_\lambda$, then the output prediction for input image $x$ can be expressed as the expectation over the allowable transformations, $\overline{g}(x,\theta,\lambda)\!:=\!\mathbb{E}_{T\!\sim\!T_\lambda}[g(T x, \theta)]$ to incorporate the invariant transformations~\cite{invar1,invar2,invar3}, where $g(x;\theta)$ is the classification scores (logits) of the neural network.
Using the Lie group algebra, one can have differentiable transformations $T_\lambda\!=\!\textrm{expm}\big( \sum_{i=1}^6 \epsilon_i \lambda_i G_i \big)$ where $\epsilon_{1:6}$ are uniform random samples from $[-1,1]$, and $G_{1:6}$ are known $(3 \times 3)$ generating matrices for X-/Y-translations, rotation, X-/Y-scaling and shearing~\cite{liegroup}.  

We can then formulate the invariance learning as BLO,
\begin{align}
\min_\lambda & -\rho ||\lambda||_2^2 + \mathbb{E}_{(x,y)\sim val}[l(\overline{g}(x,\theta^*(\lambda),\lambda), y)] \\
\textrm{s.t.} & \ \theta^*(\lambda) = \arg\min_\theta \mathbb{E}_{(x,y)\sim train}[l(\overline{g}(x,\theta,\lambda), y)]
\end{align}
Note that the input transformation is applied to both training and validation set images. The negative L2 norm is added to the validation loss to encourage more aggressive data augmentation ($\lambda$ being away from 0), which can be helpful for better generalization. The coefficient $\rho=0.01$ is chosen by cross validation for simplicity, otherwise one could form another level of meta learning tri-level optimization, which is out of the scope. 
The above formulation is similar to that of Augerino~\cite{augerino}, however, they do not form a BLO, but only optimize the outer loss for both training and validation sets jointly over $\lambda$ and $\theta$. 

On the input image transformed MNIST datasets (e.g., randomly rotated by 180 degrees (full) or 90 (partial)), our test results are shown in Table~\ref{apptab:auger_mnist}. Similar to competing methods, we approximate $\overline{g}$ by the Monte Carlo average of 31 samples $T\!\sim\!T_\lambda$. We split the training data into $80\%/20\%$ train/validation sets for meta learning. 
Our approach outperforms Augerino~\cite{augerino}, 
and being comparable to LILA~\cite{lila} which is computationally expensive to perform full-Hessian Laplace approximation following the Bayesian model selection strategy.
For the full-rotated case (Rot-180), our model returns the optimal geometric transformation hyperparameters: 
$\lambda^*\!=\![-0.01, 0.01, -10.0, -0.00, 0.01, -0.06]$  that correspond to X-/Y-translation, rotation, X-/Y-scale, and shear parameters, respectively. 
For partial rotation (Rot-90), 
$\lambda^*\!=\![0.02, 0.03, 1.04, 0.02, 0.06, 0.03]$. 
That is, the rotation parameter is correctly identified, successfully learning the input invariance that resides in the data.

Following~\cite{lila}, the main backbone network is the MLP with a single hidden layer of 1000 units and $tanh$ nonlinearity. We use the inner-loop iterations $(B=4,M=1)$ with the temperature $\tau=0.01$. The learning rates are 0.01 for the outer loop and 0.05 for the inner loop. The batch size is 1000 and the outer loop is iterated up to 2000 iterations. 

\section{Additional Experimental Results}\label{appsec:additional_expmt}

\subsection{Running Times}\label{appsec:wall_clock}

We analyze empirical computation times of the proposed method by measuring the wall clock time comparison for the small-sized Synth-1D dataset and the large-scale L1-Reg dataset (with ViT backbones). The average running time for one outer iteration (hypergradient update) is summarized in  Table~\ref{apptab:wall_clock}. We use a single RTX-2080 GPU for Synth-1D and a single V100 GPU for L1-Reg experiments.

\begin{table}[t!]
\setlength{\tabcolsep}{2.7pt}
\caption{Wall clock running times for Synth-1D and L1-Reg problems. 
}
\vspace{-1.0em}
\centering
\begin{scriptsize}
\centering
\begin{tabular}{cccccc}
\toprule
Time (msec) & HPO-SGLD (Ours) & IFT & IFT-CG & RMD & FMD
\\
\hline
Synth-1D\Tstrut & $179.8$ & $76.7$ & $90.7$ & $52.0$ & $164.3$ \\
\midrule
L1-Reg\Tstrut & $2325.3$ & $2175.3$ & $1976.0$ & $1069.0$ & infeasible \\
\bottomrule
\end{tabular}
%
\end{scriptsize}
\label{apptab:wall_clock}
\end{table}

Our theoretical complexity analysis (Table~\ref{tab:complexity} in our main paper) shows that our HPO-SGLD is more scalable than unroll-based RMD and FMD algorithms. However, our current implementation introduces constant factors that make it slower in wall-clock time than IFT and RMD on moderate sized tasks. This is largely because we haven't utilized Hessian-vector product operations optimized for our specific purpose, but rather relied on a general-purpose library, namely PyTorch's \texttt{autograd.grad()}. A more sophisticated Hessian-vector computation module (e.g., defining our own backward method) would improve this wall-clock time.

\subsection{Comparison with Other Stochastic Methods for Bi-Level optimization}\label{appsec:other_stochastic_blo}

In this section we differentiate our Stoachastic BLO algorithm from some existing works~\cite{ref_A,ref_B,ref_C,ref_D} on stochastic methods for BLO in the literature.

Some of the works consider and deal with stochastic noise in the BLO problem. To clarify, by stochasticity, they mostly meant {\em noisy inner functions} (e.g., minibatch loss evaluation), and performed convergence analysis based on the unbiasedness of the estimates. We emphasize that as far as we know, no prior work introduced the stochastic optimization by turning the inner loss function into the Gibbs expectation form as we did. This means that their hypergradient estimates are not robust to multiple inner minima as in our HPO-SGLD.

About~\cite{ref_B}: The method (AmIGO) is very similar to IFT-CG (aka iMAML~\cite{imaml}) -- basically forming a quadratic function $0.5 \cdot z^\top H z - b^\top z$ where $H$ is the Hessian in the IFT. The optimal solution $z^* = H^{-1} b$ requires Hessian inverse, and IFT-CG does conjugate gradient to iteratively find $z^*$. In the AmIGO, they used warm-start vanilla gradient descent to get $z^*$ (They also have a CG version, AmIGO-CG in~\cite{ref_B}). Although they showed convergence analysis for a convex inner function case, the performance is considerably different from IFT-CG as we demonstrated in the empirical comparison below.

The work~\cite{ref_D} is also similar to~\cite{ref_B}, but instead of finding the inner optima $\theta^*(\lambda)$ and optimum of the quadratic $z^*$ by several iterations, \cite{ref_D} used one-step update and repeated the procedure many times. 
\cite{ref_C} follow the framework of \cite{ref_D}, and used STORM variance reduction technique~\cite{storm}. \cite{ref_A} is basically IFT (Neumann) but adopted momentum acceleration in the inner loop optimization, and also used STORM-like variance reduction technique. 
IFT-Neumann was also introduced/used in the BLO community~\cite{ghadimi18,blo_naa}.

In summary: \cite{ref_B,ref_C,ref_D} are all about IFT with an auxiliary quadratic function to iteratively solve Hessian inversion. Hence they inherit the main drawbacks from IFT: (i) difficult to handle the $\lambda$-dependent initial $\theta^{(0)}$ problems (e.g., MAML-like initial model learning) in general; and (ii) a strong assumption of the stationarity condition for the inner problem, with unclear impact of violating that assumption in practice. Finally, most of their convergence analysis is based on the convexity of the inner problem. This is violated in practice for applications in neural networks, but our HPO-SGLD convergence analysis does not rely on this strong assumption.

\textbf{Empirical comparison.} 
We also have an empirical comparison with AmIGO from~\cite{ref_B}, one of the related methods that the reviewer mentioned. They have two versions: Conjugate-Gradient and SGD (denoted by AmIGO-CG/SGD, respectively) depending on how the Hessian inverse problem is solved in a warm-start manner. The results are summarized in Table~\ref{apptab:amigo}.
%
\begin{table}[t!]
\setlength{\tabcolsep}{5.7pt}
\caption{Comparison with AmIGO~\cite{ref_B}, another stochastic methods for BLO. 
}
\vspace{-1.0em}
\centering
\begin{scriptsize}
\centering
\begin{tabular}{cccccc}
\toprule
Errors & Synth-1D & Noisy-Synth-1D & L1-Reg (Pets) & L1-Reg (DTD) & L1-Reg (Flowers)
\\
\hline
IFT-Neumann\Tstrut & $(0.0004, 0.0004)$ & $(0.0153, 0.0426)$ & $28.07 \pm 0.65$ & $34.89 \pm 0.26$ & $15.30 \pm 0.19$ \\
\midrule
IFT-CG\Tstrut & $(0.0001, 0.0001)$ & $(0.0057, 0.0311)$ & $27.47 \pm 0.90$ & $34.80 \pm 0.31$ & $15.15 \pm 0.15$ \\
\midrule
RMD\Tstrut & $(0.0127, 0.0141)$ & $(0.0124, 0.0150)$ & $25.49 \pm 0.64$ & $36.36 \pm 0.48$ & $15.65 \pm 0.37$ \\
\midrule
FMD\Tstrut & $(0.0127, 0.0141)$ & $(0.0123, 0.0149)$ & Infeasible & Infeasible & Infeasible \\
\midrule
AmIGO-SGD~\cite{ref_B}\Tstrut & $(0.0001, 0.0001)$ & $(0.0008, 0.0233)$ & $27.41 \pm 1.04$ & $34.70 \pm 0.39$ & $15.65 \pm 0.35$ \\
\midrule
AmIGO-CG~\cite{ref_B}\Tstrut & $(0.0001, 0.0001)$ & $(0.0057, 0.0311)$ & $26.57 \pm 1.15$ & $34.86 \pm 0.47$ & $15.55 \pm 0.21$ \\
\midrule
HPO-SGLD (Ours)\Tstrut & $(0.0001, 0.0000)$ & $(0.0008, 0.0113)$ & $23.75 \pm 0.21$ & $33.53 \pm 0.46$ & $14.56 \pm 0.08$ \\
\bottomrule
\end{tabular}
%
\end{scriptsize}
\label{apptab:amigo}
\end{table}
%
Hyperparameters are chosen fairly for all competing methods (e.g., the numbers of outer/inner iterations are 200/100 in Synth1D). For L1Reg we only show the best test errors among different hyperparameter choices (Table 4 of the submission) for each method. AmIGO-SGD/CG take up to 10 SGD/CG iterations. Since we have observed that the CG iterations converge so quickly, there is rarely a significant difference between 0-initial cold-start IFT-CG and the warm-start AmIGO-CG.

\subsection{Synthetic Quadratic Function Experiments}\label{appsec:synth_quad}

As another proof of concept, we have conducted synthetic experiments with quadratic loss functions where both inner and outer loss functions are convex quadratic, similar to the setup in the previous work~\cite{ref_B}. We have aimed to follow their setup as much as possible from their paper since we could not find any reproducible code publicly available. 

We test two cases: the condition number of the inner quadratic matrix is either 10 (easier) or 1000 (harder). For each competing method, we measure the number of iterations required to achieve $10^{-5}$ relative error to the true optimal solution. 

Here we also include the results of AmIGO-CG/SGD from~\cite{ref_B} as the paper offered convergence analysis for strong convex inner loss function cases. They have two options CG and SGD versions to solve the Hessian inverse problem in a warm-start manner. 
Some hyperparameters are as follows: the maximum number of inner iterations is 100 (in our HPO-SGLD, we set $B=M=50$ for fairness); we run all methods up to 2000 outer iterations, hence the maximum number of total iterations is $200K$. The outer learning rate is $0.1$ for all methods, and for those that need inner learning rate (e.g., our HPO-SGLD, RMD, and AmIGO-SGD), we choose 0.05. HPO-SGLD temperature is 1, IFT-Neumann has ($\alpha=0.99$, series length 10), IFT-CG, AmIGO-CG/SGD all use the maximum number of iterations 10 to solve the Hessian inverse. 

\begin{table}[t!]
\setlength{\tabcolsep}{2.7pt}
\caption{Total number of iterations to achieve $<10^{-5}$ relative error for synthetic quadratic experiments. 
}
\vspace{-1.0em}
\centering
\begin{scriptsize}
\centering
\begin{tabular}{ccccccc}
\toprule
Condition number & HPO-SGLD (Ours) & IFT-Neumann & IFT-CG & AmIGO-CG~\cite{ref_B} & AmIGO-SGD~\cite{ref_B} & RMD
\\
\hline
10\Tstrut & $12K$ & $13K$ & $12K$ & $12K$ & $12K$ & $13K$ \\
\midrule
1000\Tstrut & $58K$ & $57K$ & $56K$ & $52K$ & $58K$ & $>200K$ \\
\bottomrule
\end{tabular}
%
\end{scriptsize}
\label{apptab:quadratic}
\end{table}

The averaged results (the total number of iterations for $<10^{-5}$ relative error) are summarized in Table~\ref{apptab:quadratic}.
The results imply that our HPO-SGLD exhibits convergence speed comparable to the IFT-based conjugate gradient methods even for these well-behaved single-mode convex BLO problems, both in numerically easy (condition number 10) and hard (1000) settings.

\subsection{Comparison with Evolutionary Search on Synth-1D Datasets}\label{appsec:synth1d_es}

Although the Evolutionary Search method does not scale well with the dimension of the hyperparameters $\lambda$, we have implemented a version of the evolutionary search algorithm to solve the BLO problem for the small Synth-1D experiments. The idea is to estimate the hypergradient using the ES-gradient~\cite{es_ns} with a smoothed version of the outer loss. More specifically, the ES-gradient of a function $f(x)$ is $\tilde{\nabla}f(x) = \frac{1}{\sigma} \cdot E_{z\sim N(0,I)} [ z \cdot f(x+\sigma z) ]$.  We apply it to BLO’s outer loss function $L_V(\lambda, \theta^*(\lambda))$. We choose $\sigma=0.001$ and number of ES samples 100. The results are summarized in Table~\ref{apptab:synth1d_es}. 

\begin{table}[t!]
\setlength{\tabcolsep}{4.7pt}
\caption{Comparison with the Evolutionary Search method~\cite{es_ns} on Synth-1D datasets. 
}
\vspace{-1.0em}
\centering
\begin{scriptsize}
\centering
\begin{tabular}{ccc}
\toprule
Errors ($\lambda$, $\theta$) & Synth-1D & Noisy-Synth-1D
\\
\hline
IFT-Neumann\Tstrut & $(0.0004, 0.0004)$ & $(0.0153, 0.0426)$ \\
\midrule
IFT-CG\Tstrut & $(0.0001, 0.0001)$ & $(0.0057, 0.0311)$ \\
\midrule
RMD\Tstrut & $(0.0127, 0.0141)$ & $(0.0124, 0.0150)$ \\
\midrule
FMD\Tstrut & $(0.0127, 0.0141)$ & $(0.0123, 0.0149)$ \\
\midrule
Evolutionary Search~\cite{es_ns}\Tstrut & $(0.0275, 0.0169)$ & $(0.0041, 0.0173)$ \\
\midrule
HPO-SGLD (Ours)\Tstrut & $(0.0001, 0.0000)$ & $(0.0008, 0.0113)$ \\
\bottomrule
\end{tabular}
%
\end{scriptsize}
\label{apptab:synth1d_es}
\vspace{-1.0em}
\end{table}

\end{document}